\documentclass{article}
\usepackage{fullpage}
\usepackage{color}
\usepackage[table]{xcolor} 

\usepackage{microtype}
\usepackage{graphicx}
\usepackage{subfig}
\usepackage{booktabs} 
\usepackage{amsmath}
\usepackage{amsthm}
\usepackage{amssymb}
\usepackage{enumerate}
\usepackage{enumerate}
\usepackage{paralist}
\usepackage{bm}
\usepackage{graphicx}
\usepackage{verbatim}
\usepackage{nicefrac}
\usepackage{tikz-qtree,tikz-qtree-compat}
\usepackage{tikz}
\usepackage{mathdots}
\usepackage{wrapfig}
\newcommand{\onenorm}[1]{\left\lVert#1\right\rVert_1}

\newcommand{\R}{\mathbb{R}}

\newcommand{\cX}{\mathcal{X}}

\newcommand{\E}{\mathbb{E}}

\newcommand{\reals}{\mathbb{R}}

\newcommand{\vectx}{\mathbf{x}}
\newcommand{\vecty}{\mathbf{y}}
\newcommand{\vectmu}{\boldsymbol{\mu}}
\newcommand{\norm}[1]{\left\lVert#1\right\rVert}
\newcommand{\inner}[2]{\langle#1,#2\rangle}
\DeclareMathOperator*{\argmin}{arg\,min}

\newcommand{\cupdot}{\mathbin{\mathaccent\cdot\cup}}
\newcommand{\cost}{\mathrm{cost}}
\newcommand{\diam}{\mathrm{diam}}

\definecolor{cluster0color}{RGB}{162,201,221}
\definecolor{cluster1color}{RGB}{50,158,43}
\definecolor{cluster2color}{RGB}{233,176,102}
\definecolor{cluster3color}{RGB}{106,61,154}
\definecolor{cluster4color}{RGB}{165,83,37}

\usepackage{hyperref}%
\definecolor{bluish-green}{HTML}{009E73}
\hypersetup{%
   breaklinks,%
   colorlinks=true,%
   urlcolor=[rgb]{0.0,0.2,0.75},%
   linkcolor=[rgb]{0.5,0.0,0.0},%
   citecolor=bluish-green,%
   filecolor=[rgb]{0,0,0.4}, anchorcolor=[rgb]={0.0,0.1,0.2}%
}

\usepackage[ruled,linesnumbered]{algorithm2e}

\newtheorem{theorem}{Theorem}[section]
\newtheorem{claim}[theorem]{Claim}
\newtheorem{fact}[theorem]{Fact}

\newtheorem{lemma}[theorem]{Lemma}

\newtheorem{proposition}[theorem]{Proposition}

\newcommand{\xcor}{\cX^{\mathsf{cor}}}
\newcommand{\xmis}{\cX^{\mathsf{mis}}}

\newcommand{\xvar}[1]{\textsf{#1}}

\newcommand{\xfunc}[1]{\texttt{#1}}

\newcommand{\xvbox}[2]{\makebox[#1][l]{#2}}

\SetAlFnt{\footnotesize}

\SetAlCapSty{xAlCapSty}

\SetCommentSty{xCommentSty}

\SetNlSty{mynlfont}{}{} 

\LinesNumbered

\SetSideCommentRight

\DontPrintSemicolon

\RestyleAlgo{algoruled}

\newcommand{\NULL}{\textsc{null}}

\author{
  Sanjoy Dasgupta\quad\quad\\
  University of California, San Diego\quad\quad\\
  \url{dasgupta@eng.ucsd.edu}\quad\quad
  \and
  Nave Frost\quad\quad\\
  Tel Aviv University\quad\quad\\
  \url{navefrost@mail.tau.ac.il}\quad\quad
  \and
    Michal Moshkovitz\\
    University of California, San Diego\\
    \url{mmoshkovitz@eng.ucsd.edu}
  \and
    Cyrus Rashtchian\\
    University of California, San Diego\\
  \url{crashtchian@eng.ucsd.edu}
}

\title{Explainable $k$-Means and $k$-Medians Clustering}
\date{}
\begin{document}
\maketitle

\begin{abstract}
Clustering is a popular form of unsupervised learning for geometric data. Unfortunately, many clustering algorithms lead to cluster assignments that are hard to explain, partially because they depend on all the features of the data in a complicated way. To improve interpretability, we consider using a small decision tree to partition a data set into clusters, so that clusters can be characterized in a straightforward manner. We study this problem from a theoretical viewpoint, measuring cluster quality by the $k$-means and $k$-medians objectives: Must there exist a tree-induced clustering whose cost is comparable to that of the best unconstrained clustering, and if so, how can it be found? In terms of negative results, we show, first, that popular top-down decision tree algorithms may lead to clusterings with arbitrarily large cost, and second, that any tree-induced clustering must in general incur an $\Omega(\log k)$ approximation factor compared to the optimal clustering. On the positive side, we design an efficient algorithm that produces explainable clusters using a tree with $k$ leaves. For two means/medians, we show that a single threshold cut suffices to achieve a constant factor approximation, and we give nearly-matching lower bounds. For general $k \geq 2$, our algorithm is an $O(k)$ approximation to the optimal $k$-medians and an $O(k^2)$ approximation to the optimal $k$-means. Prior to our work, no algorithms were known with provable guarantees independent of dimension and input size.
\end{abstract}

\section{Introduction}
A central direction in machine learning is understanding the reasoning behind decisions made by learned models~\cite{lipton2018mythos, molnar2019, murdoch2019interpretable}. Prior work on AI explainability focuses on the interpretation of a black-box model, known as {\em post-modeling} explainability~\cite{baehrens2010explain, ribeiro2018anchors}. While methods such as LIME~\cite{ribeiro2016should} or Shapley explanations~\cite{lundberg2017unified} have made progress in this direction, they do not provide direct insight into the underlying data set, and the explanations depend heavily on the given model. This has raised concerns about the applicability of current solutions, leading researchers to consider more principled approaches to interpretable methods~\cite{rudin2019stop}.

We address the challenge of developing machine learning systems that are explainable by design, starting from an {\em unlabeled} data set. Specifically, we consider {\em pre-modeling} explainability in the context of {clustering}. 
A common use of clustering is to identify patterns or discover structural properties in a data set by quantizing the unlabeled points. For instance, $k$-means clustering may be used to discover coherent groups among a supermarket's customers. While there are many good clustering algorithms, the resulting cluster assignments can be hard to understand because the clusters may be determined using all the features of the data, and there may be no concise way to explain the inclusion of a particular point in a cluster. This limits the ability of users to discern the commonalities between points within a cluster or understand why points ended up in different clusters.

\begin{figure*}[t]
    \centering
    \subfloat[Optimal $5$-means clusters]{
         \includegraphics[width=.33\textwidth]{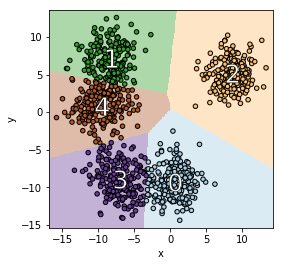}
    \label{fig:optimal_clusters}}
    \hfill
     \subfloat[Tree based $5$-means clusters]{
         \includegraphics[width=.33\textwidth]{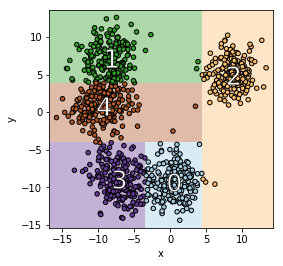}
      \label{fig:tree_clusters}
     }
     \hfill
     \subfloat[Threshold tree]{
            \resizebox{.25\textwidth}{!}{%
            \begin{tikzpicture}
            [inner/.style={shape=rectangle, rounded corners, draw, align=center, top color=white, bottom color=gray!40, scale=.85},
            leaf/.style={shape=rectangle, rounded corners, draw, align=center, scale=1},
            level 1/.style={sibling distance=4mm},
            level 2/.style={sibling distance=3mm},
            level 3/.style={sibling distance=2mm},
            level 4/.style={sibling distance=2mm},
            level distance=8mm]
            \Tree
            [.\node[inner]{$x \leq 4.5$};
                [.\node[inner]{$y \leq -4$};
                    [.\node[inner]{$x \leq -3.5$};
                        \node[leaf, top color=cluster3color!60, bottom color=cluster3color!80]{\textbf{3}};
                        \node[leaf, top color=cluster0color!60, bottom color=cluster0color!80]{\textbf{0}};    
                    ]
                    [.\node[inner]{$y \leq 4$};
                        \node[leaf, top color=cluster4color!60, bottom color=cluster4color!80]{\textbf{4}};
                        \node[leaf, top color=cluster1color!60, bottom color=cluster1color!80]{\textbf{1}};    
                    ]
                ]
                \node[leaf, top color=cluster2color!60, bottom color=cluster2color!80]{\textbf{2}};
            ]
            \node[scale=2] at (-2.5,-3.5) {$ $};
            \end{tikzpicture}
            } \label{fig:decision_tree}
     }
    \caption{
    	The optimal $5$-means clustering (left) determines uses combinations of both features. The explainable clustering (middle) uses axis-aligned rectangles summarized by the threshold tree (right). Because the clusters contain nearby points, a small threshold tree makes very few mistakes and leads to a good approximation. The benefit of explainability would be more apparent in higher dimensions.}
    \label{fig:optimal_vs_tree}
\end{figure*}

Our goal is to develop accurate, efficient clustering algorithms with concise explanations of the cluster assignments. There should be a simple procedure using a few features to explain why any point belongs to its cluster. Small decision trees have been identified as a canonical example of an easily explainable model~\cite{molnar2019, murdoch2019interpretable}, and
previous work on explainable clustering uses an unsupervised decision tree~\cite{bertsimas2018interpretable, fraiman2013interpretable,geurts2007inferring,ghattas2017clustering, liu2005clustering}. Each node of the binary tree iteratively partitions the data by thresholding on a single feature. We focus on finding $k$ clusters, and hence,  we use trees with $k$ leaves. Each leaf corresponds to a cluster, and the tree is as small as possible. 
We refer to such a tree as a {\em threshold tree}.

There are many benefits of using a small threshold tree to produce a clustering. Any cluster assignment is explained by computing the thresholds along the root-to-leaf path. By restricting to $k$ leaves, we ensure that each such path accesses at most $k-1$ features, independent of the data dimension. 
In general, a threshold tree provides an initial quantization of the data set, which can be combined with other methods for future learning tasks. While we consider static data sets, new data points can be easily clustered by using the tree, leading to explainable assignments.
To analyze clustering quality, we consider the $k$-means and $k$-medians objectives~\cite{macqueen, steinhaus}. The goal is to efficiently determine a set of $k$ centers that minimize either the squared $\ell_2$ or the $\ell_1$ distance, respectively, of the input vectors to their closest center. 

Figure~\ref{fig:optimal_vs_tree} provides an example of standard and explainable $k$-means clustering on the same data set. Figure~\ref{fig:optimal_clusters} on the left shows an optimal $5$-means clustering. Figure~\ref{fig:tree_clusters} in the middle shows an explainable, tree-based $5$-means clustering, determined by the tree in Figure~\ref{fig:decision_tree} on the right. The tree has five leaf nodes, and vectors are assigned to clusters based on the thresholds. Geometrically, the tree defines a set of axis-aligned cuts that determine the clusters. While the two clusterings are very similar, using the threshold tree leads to easy explanations, whereas using a standard $k$-means clustering algorithm leads to more complicated clusters. The difference between the two approaches becomes more evident in higher dimensions, because standard algorithms will likely determine clusters based on all of the feature values.

To reap the benefits of explainable clusters, we must ensure that the data partition is a good approximation of the optimal clustering. While many efficient algorithms have been developed for $k$-means/medians clustering, the resulting clusters are often hard to interpret~\cite{arthur2007k,kanungo02, ostrovsky2013effectiveness, shalev2014understanding}. For example, Lloyd's algorithm alternates between determining the best center for the clusters and reassigning points to the closest center~\cite{lloyd1982least}. The resulting set of centers depends in a complex way to the other points in the data set. Therefore, the relationship between a point and its nearest center may be the result of an opaque combination of many feature values. This issue persists even after dimension reduction or feature selection, because a non-explainable clustering algorithm is often invoked on the modified data set. As our focus is on pre-modeling explanability, we aim for simple explanations that use the original feature vectors.

Even though Figure~\ref{fig:optimal_vs_tree} depicts a situation in which the optimal clustering is very well approximated by one that is induced by a tree, it is not clear whether this would be possible in general. Our first technical challenge is to understand the {\em price of explainability} in the context of clustering: that is, the multiplicative blowup in $k$-means (or $k$-medians) cost that is inevitable if we force our final clustering to have a highly constrained, interpretable, form. The second challenge is to actually find such a tree {\em efficiently}. This is non-trivial because it requires a careful, rather than random, choice of a subset of features. As we will see, the kind of analysis that is ultimately needed is quite novel even given the vast existing literature on clustering.

\subsection{Our contributions}

We provide several new theoretical results on explainable $k$-means and $k$-medians clustering. Our new algorithms and lower bounds are summarized in Table~\ref{tab:results_summary}.  

\medskip \noindent {\bf Basic limitations.} 
A partition into $k$ clusters can be realized by a binary threshold tree with $k-1$ internal splits. This uses at most $k-1$ features, but is it possible to use even fewer, say $\log k$ features? In Section~\ref{sec:motivating-examples}, we demonstrate a simple data set that requires $\Omega(k)$ features to achieve a explainable clustering with bounded approximation ratio compared to the optimal $k$-means/medians clustering. In particular, the depth of the tree might need to be $k-1$ in the worst case.

One idea for building a tree is to begin with a good $k$-means (or $k$-medians) clustering, use it to label all the points, and then apply a supervised decision tree algorithm that attempts to capture this labeling. In Section~\ref{sec:standard-dt-bad}, we show that standard decision tree algorithms, such as ID3, may produce clusterings with arbitrarily high cost. Thus, existing splitting criteria are not suitable for finding a low-cost clustering, and other algorithms are needed.

\medskip \noindent  {\bf New algorithms.} 
On the positive side, we provide efficient algorithms to find a small threshold tree that comes with provable guarantees on the cost. We note that using a small number of clusters is preferable for easy interpretations, and therefore $k$ is often relatively small.
For the special case of two clusters ($k=2$), we show (Theorem~\ref{thm:optimal_2_median_means}) that a single threshold cut provides a constant-factor approximation to the optimal $2$-medians/means clustering, with a closely-matching lower bound (Theorem~\ref{clm:2_median_lower_bound}), and we provide an efficient algorithm for finding the best cut. For general $k$, we show how to approximate any clustering by using a threshold tree with $k$ leaves (Algorithm~\ref{algo:imm}). The main idea is to minimize the number of mistakes made at each node in the tree, where a mistake occurs when a threshold separates a point from its original center. Overall, the cost of the explainable clustering will be close to the original cost up to a factor that depends on the tree depth (Theorem~\ref{thm:main-k}). In the worst-case, we achieve an approximation factor of $O(k^2)$ for $k$-means and $O(k)$ for $k$-medians compared to the cost of any clustering (e.g., the optimal cost). These results do not depend on the dimension or input size; hence, we get a constant factor approximation when $k$ is constant. 

\paragraph{Approximation lower bounds.}
Since our upper bounds depend on $k$, it is natural to wonder whether it is possible to achieve a constant-factor approximation, or whether the cost of explainability grows with~$k$. On the negative side, we identify a data set such that any threshold tree with $k$ leaves must incur an $\Omega(\log k)$-approximation for both $k$-medians and $k$-means (Theorem~\ref{thm:lb-k}). 
For this data set, our algorithm achieves a nearly matching bound for $k$-medians.\vspace{2ex}

\begin{table}[!htb]
\renewcommand{\arraystretch}{1.6}
\centering
\begin{minipage}{.8\textwidth}
\centering
    \begin{tabular}{|c|cc|cc|}
        \hline
        \rowcolor[HTML]{F1F1F1}
        & \multicolumn{2}{c|}{\textbf{$k$-medians}} &
        \multicolumn{2}{c|}{\textbf{$k$-means}} \\
        \rowcolor[HTML]{F1F1F1}
        & $k = 2$ & $k > 2$ & $k = 2$ & $k > 2$  \\
         \hline
        \cellcolor[HTML]{F1F1F1}
        \textbf{Upper Bound} & {$2$} & {$O(k)$} & {$4$} & {$O(k^2)$} \\
        \cellcolor[HTML]{F1F1F1} \textbf{Lower Bound} & {$2 - \frac{1}{d}$} & {$\Omega(\log k)$} &  {$3\left(1 - \frac{1}{d} \right)^2$} & {$\Omega(\log k)$} \\
        \hline
    \end{tabular}
    \caption{Summary of our upper and lower bounds on approximating the optimal $k$-medians/means clustering with explainable, tree-based clusters. The values express the factor increase compared to the optimal solution in the worst case.
    }
    \label{tab:results_summary}
\end{minipage}
\end{table}

\subsection{Related work}
\label{sec:related}

The majority of work on explainable methods considers supervised learning, and in particular, explaining predictions of neural networks and other trained models~\cite{
alvarez2019weight, deutch2019constraints, garreau2020explaining, kauffmann2019clustering,
lipton2018mythos,lundberg2017unified, molnar2019, murdoch2019interpretable, ribeiro2016should, ribeiro2018anchors, rudin2019stop, sokol2020limetree}. In contrast, there is much less work on explainable unsupervised learning. Standard algorithms for $k$-medians/means use iterative algorithms to produce a good approximate clustering, but this leads to complicated clusters that depend on subtle properties of the data set~\cite{aggarwal09,arthur2007k, kanungo02, ostrovsky2013effectiveness}. Several papers consider the use of decision trees for explainable clustering~\cite{bertsimas2018interpretable, fraiman2013interpretable,geurts2007inferring,ghattas2017clustering, liu2005clustering}. However, all prior work on this topic is empirical, without any theoretical analysis of quality compared to the optimal clustering. We also remark that the previous results on tree-based clustering have not considered the $k$-medians/means objectives for evaluating the quality of the clustering, which is the focus of our work. It is NP-hard to find the optimal $k$-means  clustering~\cite{aloise2009np, dasgupta2008hardness} or even a very close approximation~\cite{awasthi2015hardness}. In other words, we expect tree-based clustering algorithms to incur an approximation factor bounded away from one compared to the optimal clustering.

One way to cluster based on few features is to use dimensionality reduction.
Two main types of dimensionality reduction methods have been investigated for $k$-medians/means. Work on {\em feature selection} shows that it is possible to cluster based on $\Theta(k)$ features and obtain a constant factor approximation for $k$-means/medians~\cite{boutsidis2009unsupervised, cohen2015dimensionality}. However, after selecting the features, these methods employ existing approximation algorithms to find a good clustering, and hence, the cluster assignments are not explainable. Work on {\em feature extraction} shows that it is possible to use the Johnson-Lindenstrauss transform to $\Theta(\log k)$ dimensions, while preserving the clustering cost~\cite{becchetti2019oblivious, makarychev2019performance}. Again, this relies on running a $k$-means/medians algorithm after projecting to the low dimensional subspace. The resulting clusters are not explainable, and moreover, the features are arbitrary linear combinations of the original features.

Besides explainability, many other clustering variants have received recent attention, such as fair clustering~\cite{ahmadian2020fair, backurs2019scalable, bera2019fair,chiplunkar2020solve,  huang2019coresets,kleindessner2019fair, mahabadi2020individual, schmidt2019fair}, online clustering~\cite{bhaskara20a, cohen2019online,hess2019sequential, liberty2016algorithm, moshkovitz2019unexpected}, and the use of same-cluster queries~\cite{ailon2018approximate, ashtiani2016clustering, huleihel2019same, mazumdar2017clustering}. An interesting avenue for future work would be to further develop tree-based clustering methods by additionally incorporating some of these other constraints or objectives.

\section{Preliminaries}

Throughout we use bold variables for vectors, and we use non-bold for scalars such as feature values.
Given a set of points $\cX=\{\vectx^1,\ldots,\vectx^n\}\subseteq\mathbb{R}^d$ and an integer $k$ the goal of $k$-medians and $k$-means clustering is to partition $\cX$ into $k$ subsets and minimize the distances of the points to the centers of the clusters. It is known that the optimal centers correspond to means or medians of the clusters, respectively. Denoting the centers as  $\vectmu^1,\ldots,\vectmu^k$, the aim of $k$-means is to find a clustering that minimizes the following objective
$$\cost_2(\vectmu^1, \ldots, \vectmu^k)=\sum_{\vectx\in \cX} \norm{\vectx-c_2(\vectx)}^2_2,$$ where $c_2(\vectx)=\argmin_{\vectmu \in \{\vectmu^1,\ldots,\vectmu^k\}}{\norm{\vectmu - \vectx}_2}$.
Similarly, the goal of $k$-medians is to minimize 
$$\cost_1(\vectmu^1, \ldots, \vectmu^k)=\sum_{\vectx\in \cX} \norm{\vectx-c_1(\vectx)}_1,$$ where $c_1(\vectx)=\argmin_{\vectmu \in \{\vectmu^1,\ldots,\vectmu^k\}}{\norm{\vectmu - \vectx}_1}$. 
As it will be clear from context whether we are talking about $k$-medians or $k$-means, we abuse notation and write $\cost$ and $c(\vectx)$ for brevity. We also fix the data set and use $opt$ to denote the optimal $k$-medians/means clustering, where the optimal centers are the medians or means of the clusters, respectively; hence, $\cost(opt)$ refers to the cost of the optimal $k$-medians/means clustering.

\subsection{Clustering using threshold trees}

Perhaps the simplest way to define two clusters is to use a \emph{threshold cut}, which partitions the data based on a threshold for a single feature. More formally, the two clusters can be written as $\widehat C^{\theta,i}=(\widehat C^1, \widehat C^2)$, which is defined using a coordinate $i$ and a threshold $\theta\in\mathbb{R}$ in the following way. For each input point $\vectx\in \cX$, we place $\vectx=[x_1, \ldots, x_d]$ in the first cluster $\widehat C^1$ if $x_i\leq\theta$, and otherwise $\vectx \in \widehat C^2$.
A threshold cut can be used to explain $2$-means or $2$-medians clustering because a single feature and threshold determines the division of the data set into exactly two clusters.

For $k > 2$ clusters, we consider iteratively using threshold cuts as the basis for the cluster explanations. More precisely, 
we construct a binary \emph{threshold tree}. This tree is an unsupervised variant of a decision tree. Each internal node contains a single feature and threshold, which iteratively partitions the data, leading to clusters determined by the vectors that reach the leaves. We focus on trees with exactly $k$ leaves, one for each cluster $\{1,2,\ldots, k\}$, which also limits the depth and total number of features to at most $k-1$. 

When clustering using such a tree, it is easy to understand why $\vectx$ was assigned to its cluster: we may simply inspect the threshold conditions on the root-to-leaf path for $\vectx$. This also ensures the number of conditions for the cluster assignment is rather small, which is crucial for interpretability. These tree-based explanations are especially useful in high-dimensional space, when the number of clusters is much smaller than the input dimension ($k \ll d$). 
More formally, a threshold tree $T$ with $k$ leaves induces a $k$-clustering of the data. Denoting these clusters as $\widehat{C}^j \subseteq \cX$, we define the $k$-medians/means cost of the tree as
$$
\cost_1(T) = \sum_{j=1}^k \sum_{x \in \widehat{C}^j} \|x - \mbox{median}(\widehat{C}^j) \|_1
\qquad \mbox{} \qquad 
\cost_2(T) = \sum_{j=1}^k \sum_{x \in \widehat{C}^j} \|x - \mbox{mean}(\widehat{C}^j) \|_2^2 
$$
Our goal is to understand when it is possible to efficiently produce a tree $T$ such that $\cost(T)$ is not too large compared to the optimal $k$-medians/means cost. Specifically, we say that an algorithm is an {\em $a$-approximation}, if the cost is at most $a$ times the optimal cost, i.e., if the algorithm returns threshold tree $T$ then we have
$\cost(T) \leq a\cdot \cost(opt),$ where $opt$ denotes the optimal $k$-medians/means clustering.

\section{Motivating Examples}
\label{sec:motivating-examples}

{\bf Using $k-1$ features may be necessary.}
We start with a simple but important bound showing that trees with depth less than~$k$ (or fewer than $k-1$ features) can be arbitrarily worse than the optimal clustering. Consider the data set consisting of the $k-1$ standard basis vectors $\mathbf{e}^1,\ldots,\mathbf{e}^{k-1} \in \R^{k-1}$ along with the all zeros vector. As this data set has $k$ points,  
the optimal $k$-median/means cost is zero, putting each point in its own cluster. 
Unfortunately,  it is easy to see that for this data, depth $k-1$ is necessary for clustering with a threshold tree.  Figure~\ref{fig:simplex} depicts an optimal tree for this data set. Shorter trees do not work because projecting onto any $k-2$ coordinates does not separate the data, as at least two points will have all zeros in these coordinates. Therefore, any tree with depth at most $k-2$ will put two points in the same cluster, leading to non-zero cost, whereas the optimal cost is zero. In other words, for this data set, caterpillar trees such as Figure~\ref{fig:simplex} are necessary and sufficient for an optimal clustering. This example also shows that $\Theta(k)$ features are tight for feature selection~\cite{cohen2015dimensionality} and provides a separation with feature extraction methods that use a linear map to only a logarithmic number of dimensions~\cite{becchetti2019oblivious, makarychev2019performance}.

\begin{figure}
    \centering
    \subfloat[Optimal threshold tree for the data set in $\R^{k-1}$ consisting of the $k-1$ standard basis vectors and the all zeros vector. Any optimal tree must use all $k-1$ features and have depth~$k-1$.]{
    \begin{minipage}[t]{0.45\textwidth}
    \centering
    \begin{tikzpicture}
	[inner/.style={shape=rectangle, rounded corners, draw, align=center, top color=white, bottom color=gray!40, scale=.85},
	dots/.style={shape=rectangle, align=center, top color=white, scale=.85},
	leaf/.style={shape=rectangle, rounded corners, draw, align=center, scale=.85},
	level 1/.style={sibling distance=2mm},
	level 2/.style={sibling distance=2mm},
	level 3/.style={sibling distance=2mm},
	level 4/.style={sibling distance=2mm},
	level distance=8mm]
	\Tree
	[.\node[inner]{$x_{i_1} \leq 0.5$};
	[.\node[inner]{$x_{i_2} \leq 0.5$};
	[.\node[inner] {$\ldots$};
	[.\node[inner] {$x_{i_d} \leq 0.5$};
	\node[leaf]{\textbf{$\mathbf{0}$}};
	\node[leaf]{\textbf{$\mathbf{e}^{i_d}$}};
	]    
	\node[leaf]{\textbf{$\ldots$}};
	]
	\node[leaf]{\textbf{$\mathbf{e}^{i_2}$}};
	]
	\node[leaf]{\textbf{$\mathbf{e}^{i_1}$}};
	]
	\node[top color=white, bottom color=white, scale=.75] at (-0.8,-4.1) {};
	\end{tikzpicture}
	\end{minipage}
	\label{fig:simplex}}
	\hfill
    \subfloat[The ID3 split results in a $3$-means/medians clustering with arbitrarily worse cost than the optimal because it places the top two points in separate clusters. Our algorithm (Section~\ref{sec:k-means}) instead starts with the optimal first split.
    ]{
    \includegraphics[width=.4\textwidth, trim=-2cm 0 -2cm 0, clip]{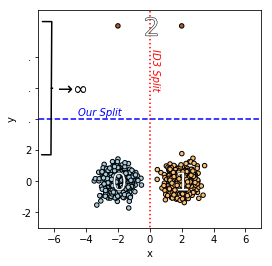}
    \label{fig:imm_vs_id3}}
\caption{Motivating examples showing that (a) threshold trees may need depth $k-1$ to determine $k$ clusters, and (b) standard decision tree algorithms such as ID3 or CART perform very badly on some data sets.}
\end{figure}

\paragraph{Standard top-down decision trees do not work.}
\label{sec:standard-dt-bad}
A natural approach to building a threshold tree is to (1) find a good $k$-medians or $k$-means clustering using a standard algorithm, then (2) use it to label all the points, and finally (3) apply a supervised decision tree learning procedure, such as ID3~\cite{quinlan1986induction,quinlan2014c4} to find a threshold tree that agrees with these cluster labels as much as possible. ID3, like other common decision tree algorithms, operates in a greedy manner, where at each step it finds the best split in terms of \emph{entropy} or \emph{information gain}. We will show that this is not a suitable strategy for clustering and that the resulting tree can have cost that is arbitrarily bad. 
In what follows, denote by $\cost({ID3}_{\ell})$  the cost of the decision tree with~$\ell$ leaves returned by ID3 algorithm.


Figure \ref{fig:imm_vs_id3} depicts a data set $\cX \subseteq \reals^2$ partitioned into three clusters $\cX = \cX_0 \cupdot \cX_1 \cupdot \cX_2$.
We define two centers $\vectmu^0=(-2,0)$ and $\vectmu^1=(2,0)$ and for each  $i\in \{0, 1\}$, we define~$\cX_i$ as $500$ i.i.d. points $\vectx \sim \mathcal{N}(\vectmu^i, \epsilon)$ for some small $\epsilon > 0$. 
Then, $\cX_2 = \{(-2, v), (2, v)\}$ where $v \to \infty$.  
With high probability, we have that the optimal $3$-means clustering is $(\cX_0, \cX_1, \cX_2)$, i.e. $\vectx \in \cX$ gets label $y\in\{0,1,2\}$ such that $\vectx \in \cX_y$.
The ID3 algorithm minimizes the entropy at each step. In the first iteration, it splits between the two large clusters. As a result $(-2, v)$ and $(2, v)$ will also be separated from one another. Since $ID3_3$ outputs a tree with exactly three leaves, one of the leaves must contain a point from $\cX_2$ together with points from either $\cX_0$ or $\cX_1$, this means that $\cost(ID3_3)= \Omega(v) \to \infty$.
Note that $\cost((\cX_1, \cX_2, \cX_3))$ does not depend on $v$, and hence, it is substantially smaller than $\cost(ID3_3)$.
Unlike ID3, the optimal threshold tree first separates $\cX_2$ from $\cX_0 \cupdot \cX_1$, and in the second split it separates $\cX_0$ and $\cX_1$. Putting the outliers in a separate cluster is necessary for an optimal clustering. It is easy to extend this example to more clusters or to when ID3 uses more leaves.

\section{Two Clusters Using a Single Threshold Cut}
\label{sec:2-means}

In this section, we consider the case of $k=2$ clusters, and we study how well a single threshold cut can approximate the optimal partition into two clusters. 

\subsection{Algorithm for \texorpdfstring{$k=2$}{k=2}}
\label{sec:2means-alg}

We present an algorithm to efficiently minimize the cost using a single threshold cut.  We begin by considering a single feature $i$ and determining the value of the best threshold $\theta \in \R$ for this feature. Then, we minimize over all features $i \in [d]$ to output the best threshold cut. We focus on the $2$-means algorithm; the $2$-medians case is similar.

For feature~$i$, we first sort the input points according to this feature, i.e., assume that the vectors are indexed as $x^1_i \leq \ldots \leq x^n_i.$
Notice that when restricting to this feature, there are only $n-1$ possible partitions of the data set into two non-empty clusters. In particular, we can calculate the cost of all threshold cuts for the $i$th feature by scanning the values in this feature from smallest to largest. 
Then, we compute for each position $p \in [n-1]$ $$\cost(p)=\sum_{j=1}^p\norm{\vectx^j-\vectmu^1(p)}_2^2 + \sum_{j=p+1}^n\norm{\vectx^j-\vectmu^2(p)}_2^2,$$
where we denote the optimal centers for these clusters as  $\vectmu^1(p)=\frac{1}{p}\sum_{j=1}^p \vectx^j$ and $\vectmu^2(p)=\frac{1}{n - p}\sum_{j=p + 1}^n \vectx^j$ because these are the means of the first $p$ and last $n-p$ points, respectively. Because there are $O(nd)$ possible thresholds, and naively computing the cost of each requires time $O(nd)$, this would lead to a running time of $O(n^2d^2)$. 
We can improve the time to $O(nd^2 + nd\log n)$ by using dynamic programming. Pseudo-code for the algorithm and description of the dynamic programming are in Appendix~\ref{sec:DP_Implementation}.

\subsection{Theoretical guarantees for \texorpdfstring{$k=2$}{}}

We prove that there always exists a threshold cut with low cost. Since our algorithm from the previous section finds the {\em best} cut, it achieves the guarantees of this theorem.
\begin{theorem}\label{thm:optimal_2_median_means}
	For any data set $\cX \subseteq \R^d$, there is a threshold cut $\widehat C$ such that the $2$-medians cost satisfies
	$$\cost(\widehat C)\leq 2 \cdot \cost(opt),$$
	and there is a threshold cut $\widehat C$ such that the $2$-means cost satisfies
	$$\cost(\widehat C)\leq 4 \cdot \cost(opt),$$
	where $opt$ is the optimal 2-medians or means clustering.
\end{theorem}

The key idea of the analysis is to bound the cost of the threshold clustering in terms of the number of points on which it disagrees with an optimal clustering. Intuitively, if any threshold cut must lead to a fairly different clustering, then the cost of the optimal 2-medians/means clustering must also be large.

We note that it is possible to prove a slightly weaker bound by using the midpoint (for each feature) between the centers. When there are $t$ changes, using the midpoint shows that $\cost(opt)$ is at least $t$ times {\em half} of the distance between the two centers. In other words, this argument only captures half of the cost. Using Hall’s theorem, we show that each change corresponds to a pair in the matching, and each such pair contributes to $\cost(opt)$ the distance between the centers (not half as before). This improves the bound by a factor of two. The proof for $2$-means is in Appendix~\ref{sec:k2_upper_bound}.

\paragraph{Notation.}
We denote the optimal clusters as $C^1$ and $C^2$ with optimal centers $\vectmu^1$ and $\vectmu^2$. Notice that we  can assume $\mu^1_i \leq \mu^2_i$ for each coordinate $i$ because negating the $i$th coordinate for all points in the dataset does not change the $2$-medians/means cost.
Assume that a single threshold partitions $\cX$ into $\widehat C^1, \widehat C^2$ such that 
$$t = \min(|C^1 \Delta \widehat C^1|, |C^1 \Delta \widehat C^2|).$$ We refer to these $t$ points as {\em changes} and assume that $t$ is the minimum possible over all threshold cuts.

If $C^1,C^2$ is an optimal $2$-medians clustering, then we prove that the cost of  $\widehat C^1, \widehat C^2$ is at most twice the optimal $2$-medians cost. Similarly, if $C^1,C^2$ is an optimal $2$-means clustering, then we prove that the cost of $\widehat C^1, \widehat C^2$ is at most four times the optimal $2$-means cost. We simply need that the threshold cut $\widehat C = (\widehat C^1, \widehat C^2)$ minimizes the number of changes $t$ compared to the optimal clusters.

We begin with a structural claim regarding the best threshold cut. This will allow us to obtain a tighter bound on the optimal $2$-medians/means cost, compared to the general $k>2$ case, in terms of the necessary number of changes. We utilize Hall's theorem on perfect matchings. 

\begin{proposition}[Hall's Theorem] Let $(P,Q)$ be a bipartite graph. If all subsets $P' \subseteq P$ have at least $|P'|$ neighbors in $Q$, then there is a matching of size $|P|$.
\end{proposition}

\begin{lemma} \label{lemma:matching}
    Let $C^1$ and $C^2$ be the optimal clustering of $\cX \subseteq \R^d$, and assume that any threshold cut requires $t$ changes.
For each $i \in [d]$, there are $t$ disjoint pairs of vectors $(\mathbf{p}^j,\mathbf{q}^j)$ in $\cX$ such that $\mathbf{p}^j\in C^1$ and $\mathbf{q}^j\in C^2$ and  $q^j_i\leq p^j_i$ for every $j \in [t]$. 
\end{lemma}
\begin{proof}
	Let $\vectmu^1$ and $\vectmu^2$ be the centers for the optimal clusters $C^1$ and $C^2$. Focus on index $i \in [d]$, and assume without loss of generality that  $\mu^1_i\leq \mu^2_i$. The $t$ pairs correspond to a matching the following bipartite graph $(P,Q)$. Let $Q = C^2$ and define $P \subseteq C^1$ as the $t$ points in $C^1$ with largest value in their $i$th coordinate. Connect $\mathbf{p} \in P$ and $\mathbf{q} \in Q$ by an edge if only if $q_i\leq p_i.$ 
	By construction, a matching with $t$ edges implies our claim.
	By Hall's theorem, we just need to prove that $P'\subseteq P$ has at least $|P'|$ neighbors.
	
	Index $P = \{\mathbf{p}^1,\ldots, \mathbf{p}^t\}$ by ascending value of $i$th coordinate, $p^1_i \leq \cdots \leq p^t_i.$  Now, notice that vertices in $P$ have nested neighborhoods: for all $j > j'$, the neighborhood of $\mathbf{p}^{j'}$ is a subset of the neighborhood of $\mathbf{p}^{j}$. It suffices to prove that $\mathbf{p}^{j}$ has at least $j$ neighbors, because this implies that any subset $P' \subseteq P$ has at least $|P'|$ neighbors, guaranteeing a matching of size $|P| = t$. Indeed, if $|P'| = b$ then we know that $\mathbf{p}^{j} \in P'$ for some $j \geq b$, implying that $P'$ has at least $j \geq b = |P'|$ neighbors.

	Assume for contradiction that $\mathbf{p}^j$ has at most $j-1$ neighbors. We argue that the threshold cut $x_i\leq p^j_i$ has fewer than $t$ changes, which contradicts the fact that all threshold cuts must make at least $t$ changes.  By our assumption, there are at most $j-1$ points that are smaller than $p^j_i$ and belong to the second cluster. 
	By the definition of $P$, there are exactly $t-j$ points with a larger $i$th coordinate than $p^j_i$ in the first cluster. Therefore, the threshold cut $x_i\leq p^j_i$ makes at most $(t-j)+(j-1)<t$ changes, a contradiction.
\end{proof}

\subsection{Upper Bound Proof for 2-medians}
	Suppose $\vectmu^1,\vectmu^2$ are optimal $2$-medians centers for clusters $C^1$ and $C^2$, and that the threshold cut $\widehat C$ makes $t$ changes, which is the minimum possible. 
	A simple argument allows us to upper bound the cost of the threshold cut as $\cost(opt)$ plus an error term that depends on the number of changes. More formally, Lemma~\ref{lemma:tree-cost} (see Section~\ref{sec:general_k_upper_bound}) in the special case of $k=2$ implies that 
	$$\cost(\widehat C) \leq \cost(opt) + t\onenorm{\vectmu^1-\vectmu^2}.$$ 
	In other words, the core of the argument is to prove that
	$t\onenorm{\vectmu^1-\vectmu^2} \leq \cost(opt).$

	Applying Lemma~\ref{lemma:matching} for each coordinate $i\in[d]$ guarantees~$t$ pairs of vectors $(\mathbf{p}^1,\mathbf{q}^1), \ldots, (\mathbf{p}^t,\mathbf{q}^t)$ with the following properties. Each $p^j_i$ corresponds to the $i$th coordinate of some  point in $C^1$ and $q^j_i$ corresponds to the $i$th coordinate of some point in $C^2$. Furthermore, for each coordinate, the~$t$ pairs correspond to $2t$ distinct points in $\cX$. Finally, we can assume without loss of generality that
	$\mu^1_i \leq \mu^2_i$ and $q^j_i \leq p^j_i$, 
	which implies that 
	\begin{eqnarray*}
	\cost(opt) \ge \sum_{i=1}^d\sum_{j=1}^t |\mu^2_i-q_i^j| + |p^j_i-\mu^1_i| &\geq&
	\sum_{i=1}^d\sum_{j=1}^t(\mu^2_i-q^j_i) + (p^j_i-\mu_i^1)
	\\ &\geq& \sum_{i=1}^d\sum_{j=1}^t(\mu^2_i-q^j_i) + (q^j_i-p^j_i) + (p^j_i-\mu_i^1)
	\\ &=& t \cdot \sum_{i=1}^d(\mu^2_i-\mu_i^1)
	=t\onenorm{\vectmu^2-\vectmu^1}.
		\end{eqnarray*}

\subsection{Lower bounds for \texorpdfstring{$k=2$}{}}

We next show that optimal clustering is not, in general, realizable with a single threshold cut, except in a small number of dimensions (e.g., $d=1$). 
Our lower bounds on the approximation ratio increase with the dimension, approaching two for $2$-medians or three for $2$-means. 

The two lower bounds are based on a data set $\cX \subseteq \mathbb{R}^d$ consisting of $2d$ points, split into two optimal clusters each with $d$ points. The first cluster contains the $d$ vectors of the form $\mathbf{1} - \mathbf{e}^i$, where $\mathbf{e}^i$ is the $i$th coordinate vector and $\mathbf{1}$ is the all-ones vector. The second cluster contains their negations, $-\mathbf{1} + \mathbf{e}^i$. Due to the zero-valued coordinate in each vector, any threshold cut must separate at least one vector from its optimal center. In the case of $2$-medians, each incorrect cluster assignment incurs a cost of $2d$. The optimal cost is roughly $2d$, while the threshold cost is roughly $4d$ (correct assignments contribute $\approx 2d$, plus $2d$ from the error), leading to an approximation ratio of nearly two. A similar result holds for $2$-means. The proof of these two lower bounds is in Appendix~\ref{sec:k2_lower_bound}.

\begin{theorem}\label{clm:2_median_lower_bound}
For any integer $d \geq 1$, define data set $\cX \subseteq \mathbb{R}^d$ as above. Any threshold cut $\widehat C$ must have $2$-medians cost 
$$\cost(\widehat C)\geq \left(2-\frac1d\right)\cdot \cost(opt)$$
and $2$-means cost
$$\cost(\widehat C)\geq 3\left(1-\frac1{d}\right)^2\cdot \cost(opt),$$
where $opt$ is the optimal 2-medians or means clustering.
\end{theorem}
\section{Threshold trees with $k > 2$ leaves}
\label{sec:k-means}

We provide an efficient algorithm to produce a threshold tree with $k$ leaves that constitutes an approximate $k$-medians or $k$-means clustering of a data set $\cX$. Our algorithm, Iterative Mistake Minimization (IMM), starts with a reference set of cluster centers, for instance from a polynomial-time constant-factor approximation algorithm for $k$-medians or $k$-means~\cite{aggarwal09}, or from a domain-specific clustering heuristic. 

We then begin the process of finding an explainable approximation to this reference clustering, in the form of a threshold tree with $k$ leaves, whose internal splits are based on single features. The way we do this is almost identical for $k$-medians and $k$-means, and the analysis is also nearly the same. Our algorithm is deterministic and its run time is only $O(kdn \log n)$, after finding the initial centers.

As discussed in Section~\ref{sec:motivating-examples}, existing decision tree algorithms use greedy criteria that are not suitable for our tree-building process. However, we show that an alternative greedy criterion---minimizing the number of {\em mistakes} at each split (the number of points separated from their corresponding cluster center)---leads to a favorable approximation ratio to the optimal $k$-medians or $k$-means cost.

\subsection{Our algorithm}
\label{sec:k-means-alg}

\begin{wrapfigure}{r}{0.5\textwidth}
\centering
\begin{minipage}{.5\textwidth}
\begin{algorithm}[H]
\SetKwFunction{Iterative Mistake Minimization}{Iterative Mistakes Minimization}
\SetKwFunction{BuildTree}{BuildTree}
\SetKwInOut{Input}{Input}\SetKwInOut{Output}{Output}\SetKwInOut{Preprocess}{Preprocess}
\Input{%
	$\vectx^1, \ldots, \vectx^n$ -- vectors in $\reals^d$\\
	$k$ -- number of clusters\\
  }
\Output{%
    root of the threshold tree
}

\LinesNumbered
\setcounter{AlgoLine}{0}
\BlankLine

$\vectmu^1, \ldots \vectmu^k \leftarrow \xfunc{k-Means}(\vectx^1, \ldots, \vectx^n, k)$\;

\ForEach {$j \in [1, \ldots, n]$}
{
    $y^j \leftarrow \argmin_{1 \leq \ell \leq k} \lVert \vectx^j - \vectmu^\ell \rVert$\;
}

\Return $\xfunc{build\_tree}(\{\vectx^j\}_{j=1}^n, \{y^j\}_{j=1}^n, \{\vectmu^j\}_{j=1}^k)$\;

\SetKwProg{buildtree}{$\xfunc{build\_tree}(\{\vectx^j\}_{j=1}^m, \{y^j\}_{j=1}^m, \{\vectmu^j\}_{j=1}^k)$:}{}{}
\LinesNumbered
\setcounter{AlgoLine}{0}
\BlankLine
\buildtree{}{
\If{$\{y^j\}_{j=1}^m$ \text{is homogeneous}}
{
    $\xvar{leaf}.cluster \leftarrow y^1$\;
    
    \Return $\xvar{leaf}$\;
}
\ForEach {$i \in [1, \ldots, d]$}
{
    $\ell_i \leftarrow \min_{1 \leq j \leq m} \mu^{y^j}_i$\;
    
    $r_i \leftarrow \max_{1 \leq j \leq m} \mu^{y^j}_i$\;
}
$i, \theta \leftarrow \argmin_{i,\ell_i \leq \theta < r_i} \sum_{j=1}^m \xfunc{mistake}(\vectx^j, \vectmu^{y^j}, i, \theta)$\;\label{ln:k_dynamic}

$\xvar{M} \leftarrow \{j \mid \xfunc{mistake}(\vectx^j, \vectmu^{y^j}, i, \theta) = 1\}_{j=1}^m$\; 

$\xvar{L} \leftarrow \{j \mid (x^j_i \leq \theta) \wedge (j \not \in \xvar{M})\}_{j=1}^m$\;

$\xvar{R} \leftarrow \{j \mid (x^j_i > \theta) \wedge (j \not \in \xvar{M})\}_{j=1}^m$\;

$\xvar{node}.condition \leftarrow ``x_i \leq \theta"$\;

$\xvar{node}.lt \leftarrow \xfunc{build\_tree}(\{\vectx^j\}_{j \in \xvar{L}}, \{y^j\}_{j \in \xvar{L}}, \{\vectmu^j\}_{j=1}^k)$\;

$\xvar{node}.rt \leftarrow \xfunc{build\_tree}(\{\vectx^j\}_{j \in \xvar{R}}, \{y^j\}_{j \in \xvar{R}}, \{\vectmu^j\}_{j=1}^k)$\;

\Return $\xvar{node}$\;
}

\SetKwProg{mistake}{$\xfunc{mistake}(\vectx, \vectmu, i, \theta)$:}{}{}
\LinesNumbered
\setcounter{AlgoLine}{0}
\BlankLine
\mistake{}{
\Return $(x_i \leq \theta) \neq (\mu_i \leq \theta)$  ? $1$ : $0$\;
}
\caption{\textsc{\newline Iterative Mistake Minimization}}
\label{algo:imm}
\end{algorithm}
\end{minipage} \vspace{-2ex}
\end{wrapfigure}

Algorithm~\ref{algo:imm} takes as input a data set $\cX \subseteq \R^d$. The first step is to obtain a reference set of $k$ centers $\{\vectmu^1, \ldots, \vectmu^k\}$, for instance from a standard clustering algorithm. We assign each data point $\vectx^j$ the label $y^j$ of its closest center. Then, the {\tt{build\_tree}} procedure looks for a tree-induced clustering that fits these labels. The tree is built top-down, using binary splits. Each node $u$ can be associated with the portion of the input space that passes through that node, a hyper-rectangular region $\mbox{cell}(u) \subseteq \mathbb{R}^d$. If this cell contains two or more of the centers $\vectmu^j$, then it needs to be split. We do so by picking the feature $i \in [d]$ and threshold value $\theta \in \R$ such that the resulting split $x_i \leq \theta$ sends at least one center to each side and moreover produces the fewest {\em mistakes}: that is, separates the fewest points in $\cX \cap \mbox{cell}(u)$ from their corresponding centers in $\{\vectmu^j: 1 \leq j \leq k\} \cap \mbox{cell}(u)$. We do not count points whose centers lie outside $\mbox{cell}(u)$, since they are associated with mistakes in earlier splits. We find the optimal split $(i, \theta)$ by searching over all pairs efficiently using dynamic programming. We then add this node to the tree, and discard the mistakes (the points that got split from their centers) before recursing on the left and right children. We terminate at a leaf node whenever all points have the same label (i.e., a {\em homogeneous} subset). As there are $k$ different labels, the resulting tree has exactly $k$ leaves.
Figure \ref{fig:imm_example} depicts the operation of Algorithm~\ref{algo:imm}. 

We first discuss the running time, and we analyze the approximation guarantees of IMM in Section~\ref{sec:imm-approx-main}.



\smallskip \noindent {\bf Time analysis of tree building.}
We sketch how to execute the algorithm in time $O(kdn\log n)$ for an $n$-point data set. At each step of the top-down procedure, we find a coordinate and threshold pair that minimizes the mistakes at this node (line \ref{ln:k_dynamic} in \texttt{build\_tree} procedure). We use dynamic programming to avoid recomputing the cost from scratch for each potential threshold. For each coordinate $i \in [d]$, we sort the data and centers. Then, we iterate over possible thresholds. We claim that we can process each node in time $O(dn\log n)$ because each point will affect the number of mistakes at most twice. Indeed, when the threshold moves, either a data point or a center moves to the other side of the threshold. Since we know the number of mistakes from the previous threshold, we count the new mistakes efficiently as follows. If a single point switches sides, then the number of mistakes changes by at most one. If a center switches sides, which happens at most once, then we update the mistakes for this center. Overall, each point affects the mistakes at most twice (once when changing sides, and once when its center switches sides). Thus, the running time for each internal node is $O(dn\log n)$. As the tree has $k-1$ internal nodes, the total time is $O(kdn \log n)$.

\begin{figure*}[!h]
    \centering
    \begin{minipage}{.29\textwidth}
    \subfloat[Optimal $5$-means clusters]{
         \includegraphics[width=.99\textwidth]{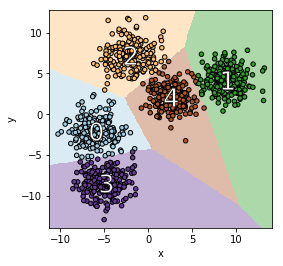}
    \label{fig:imm_example_optimal}
    }
    \end{minipage}\hspace{1em}%
    \begin{minipage}{.58\textwidth}
     \subfloat[$1^{\text{st}}$ split: 1 mistake caused by this split (1 total mistake)]{
         \includegraphics[width=.49\textwidth]{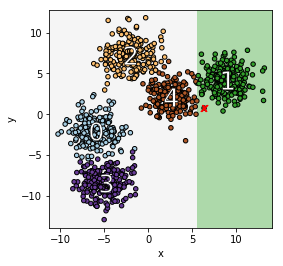}
      \label{fig:imm_example_split1}
     }\hspace{1em}%
     \subfloat[$2^{\text{nd}}$ split: 2 mistakes caused by this split (3 total mistakes)]{
         \includegraphics[width=.49\textwidth]{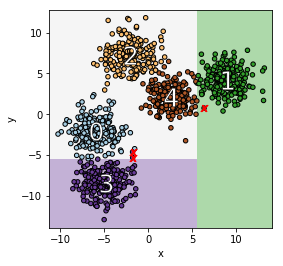}
      \label{fig:imm_example_split2}
     }\\
     \subfloat[$3^{\text{rd}}$ split: 12 mistakes caused by this split (15 total mistakes)]{
         \includegraphics[width=.49\textwidth]{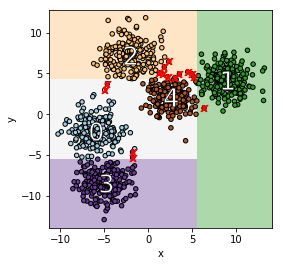}
      \label{fig:imm_example_split3}
     }\hspace{1em}%
     \subfloat[$4^{\text{th}}$ split: 0 mistakes caused by this split (15 total mistakes)]{
         \includegraphics[width=.49\textwidth]{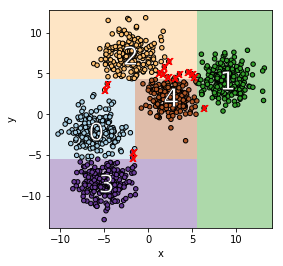}
      \label{fig:imm_example_split4}
     }
     \end{minipage}
    \caption{
    	Figure~\ref{fig:imm_example_optimal} presents the optimal $5$-means clustering. Figures~\ref{fig:imm_example_split1}--\ref{fig:imm_example_split4} depict the four splits of the IMM algorithm. The first split separates between cluster 1 and the rest, with a single mistake (marked as a red cross). Next, the IMM separates cluster 3 with 2 additional mistakes. The third split separates cluster 2, and this time the minimal number of mistakes is 12 for this split. Eventually, clusters 0 and 4 are separated without any mistakes.}
    \label{fig:imm_example}
\end{figure*}

\newpage
\subsection{Approximation guarantee for the IMM algorithm}
\label{sec:imm-approx-main}

Our main theoretical contribution is the following result.

\begin{theorem}\label{thm:main-k}\label{thm:main-k-appendix}
	Suppose that IMM takes centers $\vectmu^1,\ldots, \vectmu^k$ and returns a tree $T$ of depth $H$. Then, 
	\begin{enumerate}
		\item The $k$-medians cost is at most $$\cost(T)\leq (2H+1) \cdot \cost(\vectmu^1,\ldots, \vectmu^k)$$
		\item The $k$-means cost is at most $$\cost(T)\leq (8Hk+2) \cdot \cost(\vectmu^1,\ldots, \vectmu^k)$$
	\end{enumerate}
	In particular, IMM achieves worst case approximation factors of $O(k)$ and $O(k^2)$ by using any $O(1)$ approximation algorithm (compared to the optimal $k$-medians/means) to generate the initial centers.
\end{theorem}

We state the theorem in terms of the depth of the tree to highlight that the approximation guarantee may depend on the structure of the input data. If the optimal clusters can be easily identified by a small number of salient features, then the tree may have depth $O(\log k)$. 
We later provide a lower bound showing that an $\Omega(\log k)$ approximation factor is necessary for $k$-medians and $k$-means (Theorem~\ref{thm:lb-k}). For this data set, our algorithm produces a threshold tree with depth $O(\log k)$, and therefore, the analysis is tight for $k$-medians. We leave it as an intriguing open question whether the bound can be improved for $k$-means.

\subsubsection{Proof Overview for Theorem~\ref{thm:main-k}}

The proof proceeds in three main steps. First, we rewrite the cost of IMM in terms of the minimum number of mistakes made between the output clustering and the clustering based on the given centers. Second, we provide a lemma that relates the cost of any clustering to the number of mistakes required by a threshold clustering. Finally, we put these two together to show that the output cost is at most an $O(H)$ factor larger than the $k$-medians cost and at most an $O(Hk)$ factor larger than the $k$-means cost, respectively, where $H$ is the depth of the IMM tree, and the cost is relative to $\cost(\vectmu^1,\ldots, \vectmu^k)$. 

The approximation bound rests upon a  characterization of the excess clustering cost induced by the tree. For any internal node $u$ of the final tree $T$, let $\mbox{cell}(u) \subseteq \mathbb{R}^d$ denote the region of the input space that ends up in that node, and let $B(u)$ be the bounding box of the centers that lie in this node, or more precisely,  $B(u) = \{\vectmu^j: 1 \leq j \leq k\} \cap \mbox{cell}(u)$. We will be interested in the diameter of this bounding box, measured either by $\ell_1$ or squared $\ell_2$ norm, and denoted by $\diam_1(B(u))$ and $\diam_2^2(B(u))$, respectively.

\paragraph{Upper bounding the cost of the tree.} The first technical claim (Lemma~\ref{lemma:tree-cost}) will show that if IMM takes centers $\vectmu^1,\ldots, \vectmu^k$ and returns a tree $T$ that incurs $t_u$ mistakes at node $u \in T$, then 
\begin{itemize}
\item The $k$-medians cost of $T$ satisfies   $\displaystyle \cost(T)\leq \cost(\vectmu^1,\ldots, \vectmu^k) + \sum_{u \in T} t_u \diam_1(B(u)) $
\item The $k$-means cost of $T$ satisfies  $\displaystyle  \cost(T)\leq 2\cdot \cost(\vectmu^1,\ldots, \vectmu^k) + 2\cdot \sum_{u \in T} t_u \diam_2^2(B(u))$
\end{itemize}

Briefly, any point $\vectx$ that ends up in a different leaf from its correct center $\vectmu^j$ incurs some extra cost. To bound this, consider the internal node $u$ at which $\vectx$ is separated from $\vectmu^j$. Node $u$ also contains the center $\vectmu^i$ that ultimately ends up in the same leaf as $\vectx$. For $k$-medians, the excess cost for $\vectx$ can then be bounded by $\|\vectmu^i - \vectmu^j\|_1 \leq \diam_1(B(u))$. The argument for $k$-means is similar.

These $\sum_u t_u \diam(B(u))$ terms can in turn be bounded in terms of the cost of the reference clustering. 

\paragraph{Lower bounding the reference cost.} 
We next need to relate the cost of the centers $\vectmu^1,\ldots, \vectmu^k$ to the number of mistakes and the diameter of the cells in the tree. Lemma~\ref{aux-claim:general_k_upper_bound} will show that if IMM makes $t_u$ mistakes at node $u \in T$, then
\begin{itemize}
\item The $k$-medians cost satisfies 
$\displaystyle \sum_{u \in T} t_u \cdot \diam_1(B(u)) \leq 2H \cdot \cost(\vectmu^1,\ldots, \vectmu^k).$
\item The $k$-means cost satisfies 
$\displaystyle \sum_{u \in T} t_u \cdot  \diam_2^2(B(u)) \leq 4Hk \cdot \cost(\vectmu^1,\ldots, \vectmu^k).$
\end{itemize}

The proof for this is significantly more complicated than the upper bound mentioned above. Moreover, it contains the main new techniques in our analysis of tree-based clusterings. 

The core challenge is that we aim to lower bound the cost of the given centers using only information about the number of mistakes at each internal node. Moreover, the IMM algorithm only minimizes the {\em number} of mistakes, and not the {\em cost} of each mistake. Therefore, we must show that if every axis-aligned cut in $B(u)$ separates at least $t_u$ points~$\vectx$ from their centers, then there must be a considerable distance between the points in $\mbox{cell}(u)$ and their centers.

To prove this, we analyze the structure of points in each cell. Specifically, we consider the single-coordinate projection of points in the box $B(u)$, and we order the centers in $B(u)$ from smallest to largest for the analysis. If there are $k'$ centers in node $u$, we consider the partition of $B(u)$ into $2(k'-1)$ disjoint segments, splitting at the centers and at the midpoints between consecutive centers. Since $t_u$ is the minimum number of mistakes, we must in particular have at least $t_u$ mistakes from the threshold cut at each midpoint. We argue that each of these segments is covered at least $t_u$ times by a certain set of intervals. Specifically, we consider the intervals between mistake points and their true centers, and we say that an interval \textit{covers} a segment if the segment is contained in the interval. This allows us to capture the cost of mistakes at different distance scales. For example, if a point is very far from its true center, then it covers many disjoint segments, and we show that it also implies a large contribution to the cost. 
Claim~\ref{claim:covering} in  Section~\ref{sec:general_k_upper_bound} provides our main covering result, and we use this to argue that the cost of the given centers can be lower bounded in terms of the distance between consecutive centers in $B(u)$. For $k$-medians, we can directly derive a lower bound on the cost in terms of the $\ell_1$ diameter $\diam_1(B(u))$. For $k$-means, however, we employ Cauchy-Schwarz, which incurs an extra factor of $k$ in the bound with $\diam_2^2(B(u))$. Overall, we sum these bounds over the height $H$ of the tree, leading to the claimed upper bounds in the above lemma. 

\subsubsection{Preliminaries and Notation for Theorem~\ref{thm:main-k}}

Let  $\vectmu^1,\ldots, \vectmu^k$ be the reference centers, and let $T$ be the resulting IMM tree. Each internal node $u$ corresponds to a value $\theta_u \in \mathbb{R}$ and a coordinate $i \in [d]$. The tree partitions $\cX$ into $k$ clusters $\widehat C_1, \ldots, \widehat C_k$ based on the points that reach the $k$ leaves in $T$, where we index the clusters so that leaf $j$ contains the centers $\vectmu^j$ and $\widehat \vectmu^j$, where $\widehat \vectmu^j$ is the mean of $\widehat C_j$ for $k$-means and the median of $\widehat C_j$ for $k$-medians. This provides a bijection between old and new centers (and clusters).
Recall that the map $c:\mathcal{X} \to \{\vectmu^1,\ldots, \vectmu^k\}$ associates each point to its nearest center (i.e., $c(\vectx)$ corresponds to the cluster assignment given by the centers $\{\vectmu^1,\ldots, \vectmu^k\}$).

For a node $u \in T$, we let $\cX_u$  denote the surviving data set vectors at node $u \in T$ based on the thresholds from the root to $u$.
We also define $J_u \subseteq [k]$ be the set of surviving centers at node~$u$ from the set $\{\vectmu^1,\ldots, \vectmu^k\}$, where these centers satisfy the thresholds from the root to~$u$. 
Define $\vectmu^{L,u}$ and $\vectmu^{R,u}$ to be the maximal (smallest and largest) coordinate-wise values of the centers in $J_u$, that is, for $i \in [d]$, we set
$$
\mu^{L,u}_i = \min_{j \in J_u} \mu^j_i,
\qquad\mathrm{and}\qquad 
\mu^{R,u}_i = \max_{j \in J_u} \mu^j_i.
$$
In other words, using the previous notation and recalling that $B(u) = \{\vectmu^1,\ldots,\vectmu^k\} \cap \mathrm{cell}(u)$, we have that
$$
\mathrm{diam}_1(B(u)) = \|\vectmu^{L,u} - \vectmu^{R,u}\|_1
\qquad 
\mbox{\ and\ }
\qquad
\mathrm{diam}_2^2(B(u)) = \|\vectmu^{L,u} - \vectmu^{R,u}\|_2^2.
$$

Recall that $t_u$ for node $u \in T$ denotes the number of {\em mistakes} incurred during the threshold cut defined by $u$, where a point $\vectx$ is a mistake at node $u$ if $x$ reaches $u$, it was not a mistake before, and exactly one of the following two events occurs:
$$
\{c(\vectx)_i \leq \theta_u \ \ \mathrm{and}\ \  x_i > \theta_u\}
\qquad \mathrm{or} \qquad
\{c(\vectx)_i > \theta_u \ \ \mathrm{and}\ \   x_i \leq \theta_u\}.
$$
Let $\cX = \xcor \cup \xmis$ be a partition of the input data set into two parts, where $\vectx$ is in $\xcor$ if it reaches the same leaf node in $T$ as its center $c(\vectx)$, and otherwise, $\vectx$ is in $\xmis$. In other words, $\xmis$ contains all points $\vectx \in \cX$ that are a mistake at any node $u$ in $T$, and the rest of the points are in $\xcor$. We note that the notion of ``mistakes'' used here is different than the definition of ``changes'' used for the analysis of $2$-means/medians, even though we reuse some of the same notation.

We need a standard consequence of the Cauchy–Schwarz inequality to analyze the $k$-means cost.
\begin{claim}\label{clm:cauchy_schwarz_k_means}
	For any $a_1,\ldots,a_{m}\in\mathbb{R},$ it holds that $\sum_{i=1}^ka_i^2 \geq \frac{1}{k}\left(\sum_{i=1}^ka_i\right)^2.$
\end{claim}
\begin{proof}
	Denote by $a$ the vector $(a_1,\ldots,a_{m})$ and by $b$ the vector $(\nicefrac{1}{\sqrt{k}},\ldots,\nicefrac{1}{\sqrt{k}}).$ 
	By the Cauchy–Schwarz inequality
	$\frac{1}{k}\left(\sum_{i=1}^ka_i\right)^2=\inner{a}{b}^2\leq\sum_{i=1}^ka_i^2$
\end{proof}

We also need two facts, which state the optimal center for a cluster corresponds to mean or median of the points in the cluster, respectively. The proofs of these facts can be found in standard texts~\cite{schutze2008introduction}.

\begin{fact}\label{fact:1-means-optimal center}
For any set $S=\{\vectx^1,\ldots, \vectx^n\}\subseteq \mathbb{R}^d$, the optimal center under the $\ell_2^2$ cost  is the mean $\vectmu = \frac1n\sum_{\vectx\in S}\vectx.$
\end{fact}

\begin{fact}\label{fact:1-median-optimal center}
For any set $S=\{\vectx^1,\ldots, \vectx^n\}\subseteq \mathbb{R}^d$, the optimal center $\vectmu$ under the $\ell_1$ cost is the coordinate-wise median, defined for $i\in[d]$ as $\mu_i = \mathsf{median}(x^1_i,\ldots, x^n_i).$
\end{fact}

\subsubsection{The Two Main Lemmas and the Proof of Theorem~\ref{thm:main-k}}

To prove the theorem, we state two lemmas that aid in analyzing the cost of the given clustering versus the IMM clustering. 
The theorem will follow from these lemmas, and we will prove the lemmas in the proceeding subsections. We start with the lemma relating the number of mistakes~$t_u$ at each node $u$ and the distance between $\vectmu^{L,u}$ and $\vectmu^{R,u}$ to the cost incurred by the given centers.

\begin{lemma}\label{lemma:tree-cost}
	If IMM takes centers $\vectmu^1,\ldots, \vectmu^k$ and returns a tree $T$ of depth $H$ that incurs $t_u$ mistakes at node $u \in T$, then 
	\begin{enumerate}
		\item The $k$-medians cost of the IMM tree satisfies  $$\cost(T)\leq \cost(\vectmu^1,\ldots, \vectmu^k) + \sum_{u \in T} t_u \| \vectmu^{L,u} - \vectmu^{R,u} \|_1.$$
		\item The $k$-means cost of the IMM tree satisfies  $$\cost(T)\leq 2\cdot \cost(\vectmu^1,\ldots, \vectmu^k) + 2\cdot \sum_{u \in T} t_u \| \vectmu^{L,u} - \vectmu^{R,u} \|_2^2.$$
	\end{enumerate}
\end{lemma}

We next bound the cost of the given centers in the terms of the number of mistakes in the tree. The key idea is that if there must be many mistakes at each node, then the cost of the given centers must actually be fairly large.

\begin{lemma}\label{aux-claim:general_k_upper_bound}
	If IMM takes centers $\vectmu^1,\ldots, \vectmu^k$ and returns a tree $T$ of depth $H$ that incurs $t_u$ mistakes at node $u \in T$, then 
	\begin{enumerate}
		\item The $k$-medians cost satisfies 
		$$\sum_{u \in T} t_u \| \vectmu^{L,u} - \vectmu^{R,u} \|_1 \leq 2H \cdot \cost(\vectmu^1,\ldots, \vectmu^k).$$
		\item The $k$-means cost 
		satisfies 
		$$\sum_{u \in T} t_u \| \vectmu^{L,u} - \vectmu^{R,u} \|_2^2 \leq 4Hk \cdot \cost(\vectmu^1,\ldots, \vectmu^k).$$
	\end{enumerate}
\end{lemma}

Combining these two lemmas immediately implies Theorem~\ref{thm:main-k}.

\begin{proof}[Proof of Theorem~\ref{thm:main-k}.]
For $k$-medians, Lemmas~\ref{lemma:tree-cost} and~\ref{aux-claim:general_k_upper_bound}
together imply that 
$$\cost(T)\leq \cost(\vectmu^1,\ldots, \vectmu^k) + \sum_{u \in T} t_u \| \vectmu^{L,u} - \vectmu^{R,u} \|_1
\leq 
(2H+1) \cdot \cost(\vectmu^1,\ldots, \vectmu^k).$$
For $k$-means, we have that 
\begin{eqnarray*}
\cost(T)\leq 2\cdot \cost(\vectmu^1,\ldots, \vectmu^k) + 2\cdot \sum_{u \in T} t_u \| \vectmu^{L,u} - \vectmu^{R,u} \|_2^2
\leq (8Hk+2) \cdot \cost(\vectmu^1,\ldots, \vectmu^k).
\end{eqnarray*}

%
\end{proof}

\subsubsection{Proof of Lemma~\ref{lemma:tree-cost}}

We begin with the $k$-medians proof (the $k$-means proof will be similar). Notice that the cost can only increase when measuring the distance to the (suboptimal) center $\vectmu^j$ instead of the (optimal) center $\widehat \vectmu^j$ for cluster $\widehat C_j$, and hence,
$$
\cost(T) = \sum_{j=1}^k \sum_{\vectx \in \widehat C_j} \|\vectx - \widehat \vectmu^j\|_1
\leq 
\sum_{j=1}^k \sum_{\vectx \in \widehat C_j} \|\vectx - \vectmu^j\|_1.$$
We can rewrite this sum using the partition  $\xcor$ and $\xmis$ of $\cX$, using the fact that
whenever $\vectx\in \xcor$, then the distance is computed with respect to the true center $c(\vectx)$,
\begin{eqnarray*}
\sum_{j=1}^k \sum_{\vectx \in \widehat C_j} \|\vectx - \vectmu^j\|_1 &=& 
\sum_{j=1}^k \sum_{\vectx \in \xcor \cap \widehat C_j} \|\vectx- \vectmu^j\|_1 +
\sum_{j=1}^k \sum_{\vectx \in \xmis \cap \widehat C_j} \|\vectx- \vectmu^j\|_1
\\ &=& \sum_{\vectx \in \xcor} \|\vectx- c(\vectx)\|_1 +
\sum_{j=1}^k \sum_{\vectx \in \xmis \cap \widehat C_j} \|\vectx- \vectmu^j\|_1
\end{eqnarray*}
Starting with the above cost bound, and using the triangle inequality, we see
\begin{eqnarray*}
\cost(T) 
&\leq&
\sum_{\vectx \in \xcor} \|\vectx- c(\vectx)\|_1 +
\sum_{j=1}^k \sum_{\vectx \in \xmis \cap \widehat C_j} \|\vectx- \vectmu^j\|_1
\\ &\leq&  
\sum_{\vectx \in \xcor} \|\vectx- c(\vectx)\|_1 +
\sum_{j=1}^k \sum_{\vectx \in \xmis \cap \widehat C_j} 
(\|\vectx- c(\vectx)\|_1 + \|c(\vectx)- \vectmu^j\|_1) 
\\ &=&  \cost(\vectmu^1,\ldots, \vectmu^k) +
\sum_{j=1}^k \sum_{\vectx \in \xmis \cap \widehat C_j} \|c(\vectx)- \vectmu^j\|_1
\end{eqnarray*}

To control the second term in the final line, we must bound the cost of the mistakes. We decompose $\xmis$ based on the node $u$ where $\vectx \in \xmis$ is first separated from its true center $c(\vectx)$ due to the threshold at node~$u$.
To this end, consider some point $\vectx \in \xmis \cap \widehat C_j$, where its distance is measured to the incorrect center $\vectmu^{j} \neq c(\vectx)$.  Both centers $c(\vectx)$ and $\vectmu^j$ have survived until node $u$ in the threshold tree $T$, and hence, both vectors are part of the definitions of $\vectmu^{L,u}$ and  $\vectmu^{R,u}$. In particular, we can use the upper bound
$$\|c(\vectx) - \vectmu^j\|_1 \leq \| \vectmu^{L,u} - \vectmu^{R,u} \|_1.$$
There are $t_u$ points in $\xmis$ caused by the threshold at node $u$, and we have that
$$\sum_{j=1}^k \sum_{\vectx \in \xmis \cap \widehat C_j} \|c(\vectx) - \vectmu^j\|_1 \leq 
\sum_{u \in T} t_u \cdot \| \vectmu^{L,u} - \vectmu^{R,u} \|_1.$$
Therefore, we have, as desired
\begin{eqnarray*}
\cost(T) &\leq& \cost(\vectmu^1,\ldots, \vectmu^k) + \sum_{j=1}^k \sum_{\vectx \in \xmis \cap \widehat C_j} \|\vectx- \vectmu^j\|_1
\\&\leq& \cost(\vectmu^1,\ldots, \vectmu^k) + \sum_{u \in T} t_u \| \vectmu^{L,u} - \vectmu^{R,u} \|_1.
\end{eqnarray*}
\noindent
Analyzing $k$-means is similar; we incur a factor of two by using Claim~\ref{clm:cauchy_schwarz_k_means} instead of the triangle inequality:
\begin{eqnarray*}
\cost(T) &\leq& 
\sum_{\vectx \in \xcor} \|\vectx- c(\vectx)\|_2^2 + 
2\sum_{j=1}^k \sum_{\vectx \in \xmis \cap \widehat C_j}  (\|\vectx- c(\vectx)\|_2^2 +\|c(\vectx) - \vectmu^j\|_2^2)
\\&\leq& 2\cdot \cost(\vectmu^1,\ldots, \vectmu^k) + 2\cdot \sum_{j=1}^k \sum_{\vectx \in \xmis \cap \widehat C_j}  \|c(\vectx) - \vectmu^j\|_2^2
\\ &\leq& 2\cdot \cost(\vectmu^1,\ldots, \vectmu^k) + 2\cdot \sum_{u \in T} t_u \| \vectmu^{L,u} - \vectmu^{R,u} \|_2^2
\end{eqnarray*}

\subsubsection{Proof of Lemma~\ref{aux-claim:general_k_upper_bound}}
\label{sec:general_k_upper_bound}
	
To prove this lemma, we bound the cost at each node $u$ of tree in terms of the mistakes made at this node. For this lemma, we define $\xcor_u$ to be the set of points in $\cX$ that reach node $u$ in $T$ along with their center $c(\vectx)$. We note that $\xcor_u$ differs from $\xcor \cap \cX_u$ because a point $\vectx \in \xcor_u$ may not make it to $\xcor$ if there is a mistake later on (i.e., $\xcor$ is the union of $\xcor_u$ only over leaf nodes).

\begin{lemma}\label{lemma:mistake-bound}
For any node $u \in T$, we have that
	$$
	\sum_{\vectx \in \xcor_u} \|\vectx - c(\vectx)\|_1
	\geq 
	\frac{t_u}{2} \cdot \| \vectmu^{L,u} - \vectmu^{R,u} \|_1.
	$$
and
	$$
\sum_{\vectx \in \xcor_u} \|\vectx - c(\vectx)\|_2^2
\geq 
\frac{t_u}{4k}   \cdot  \| \vectmu^{L,u} - \vectmu^{R,u} \|_2^2.
$$
\end{lemma}
\begin{proof}
Fix a coordinate $i \in [d]$ and a node $u \in T$.
To simplify notation, we let $z_1 \leq \cdots \leq z_{k'}$ denote the {\em sorted} values of $i$th coordinate of the $k' \leq k$ centers that survive until node $u$ (so that $z_1 = \mu^{L,u}_i$ and $z_{k'} = \mu^{R,u}_i$). Observe that for each $\vectx \in \xcor_u$, the center $c(\vectx)$ must have survived until node $u$, and hence, $c(\vectx)_i$ equals one of the values $z_j$ for $j\in[k']$. 

We need a definition that allows us to relate the cost in coordinate $i$ to the distances between $z_1$ and $z_{k'}$. For consecutive values $(j,j+1)$, we say that the pair $(j,j+1)$ is {\em covered} by $\vectx$ if either 
\begin{itemize}
	\item The segment $[z_j, \frac{z_j + z_{j+1}}{2})$ is contained in the segment $[x_i,c(\vectx)_i]$, or
	\item The segment $[\frac{z_j + z_{j+1}}{2}, z_{j+1})$ is contained in the segment $[x_i,c(\vectx)_i]$.
\end{itemize} 

We prove the following claim, which enables us to relate the cost in the $i$th coordinate to the value $z_{k'} - z_1$ by decomposing this value into the distance between consecutive centers.
\begin{claim}\label{claim:covering}
For each $j = 1,2,\ldots, k'-1$, the pair $(j,j+1)$ is {covered} by at least~$t_u$ points $\vectx \in \xcor_u$.
\end{claim}
\begin{proof}
Suppose for contradiction that this does not hold. We argue that we can find a threshold value for coordinate $i$ that makes fewer than $t_u$ mistakes. To see this, assume that $(j,j+1)$ is covered by fewer than $t_u$ points $\vectx \in \cX_u$. In particular, setting the threshold to be $\frac{z_j + z_{j+1}}{2}$ separates fewer than $t_u$ points $\vectx$ from their centers $c(\vectx)$. This implies that there are fewer than $t_u$ mistakes at node $u$, which is a contradiction because the IMM algorithm chooses the coordinate and threshold pair that minimizes the number of mistakes.
\end{proof}

Now this claim suffices to prove Lemma~\ref{lemma:mistake-bound}. The only challenge is that we must string together the covering points $\vectx$ to get a bound on $z_{k'}-z_1$.

We start with the $k$-medians proof. Using the above claim, we can lower bound the contribution of coordinate $i$ to the cost of the given centers. Notice that the values $z_1 \leq \cdots \leq z_{k'}$ partition the interval between $z_1 = \mu^{L,u}_i$ and $z_{k'} = \mu^{R,u}_i$. Thus, each time $\vectx$ covers a pair $(j,j+1)$, there must be a contribution of $\frac{z_{j+1} - z_j}{2}$ to the cost $|x_i - c(\vectx)_i|$. Because each pair is covered at least $t_u$ times by Claim~\ref{claim:covering}, we conclude that 
$$
\sum_{\vectx \in \xcor_u} |x_i - c(\vectx)_i|
\geq t_u \sum_{j = 1}^{k'-1}\left(\frac{z_{j+1} - z_j}{2}\right)
= \frac{t_u}{2} (z_{k'} - z_1).
$$
To relate the bound to $\vectmu^{L,u}$ and $\vectmu^{R,u}$, we note that the above argument holds for each coordinate $i \in [d]$, and
we have that
$$
\sum_{\vectx \in \xcor_u} \|\vectx - c(\vectx)\|_1
=  \sum_{i \in [d]} \sum_{\vectx \in \xcor_u}   |\vectx_i - c(\vectx)_i|
\geq 
\frac{t_u}{2} \cdot  \| \vectmu^{L,u} - \vectmu^{R,u} \|_1.
$$

For the $k$-means proof, we apply the same argument as above, this time using Claim~\ref{clm:cauchy_schwarz_k_means} to bound the sum of squared values as 
$$\sum_{\vectx \in \xcor_u} |x_i - c(\vectx)_i|^2
\geq t_u \sum_{j = 1}^{k'-1}\left(\frac{z_{j+1} - z_j}{2}\right)^2
\geq \frac{t_u}{k} \left(\sum_{j = 1}^{k'-1}\left(\frac{z_{j+1} - z_j}{2}\right)\right)^2
= \frac{t_u}{4k} (z_{k'} - z_1)^2,$$
and therefore, summing over coordinates $i\in[d]$, we have
$$
\sum_{\vectx \in \xcor_u} \|\vectx - c(\vectx)\|_2^2
=  \sum_{i \in [d]} \sum_{\vectx \in \xcor_u}   |x_i - c(\vectx)_i|^2
\geq 
\frac{t_u}{4k} \cdot \| \vectmu^{L,u} - \vectmu^{R,u} \|_2^2.
$$
\end{proof}

\begin{proof}[Proof of Lemma~\ref{aux-claim:general_k_upper_bound}.]
	We start with the $k$-medians proof.
The factor of $H$ arises because the same points $\vectx \in \cX$ can appear in at most $H$ sets $\xcor_u$ because $H$ is the depth of the tree. More precisely, using Lemma~\ref{lemma:mistake-bound} for each node $u$, we have that 
$$
H \cdot \cost(\vectmu^1,\ldots, \vectmu^k) \geq \sum_{u \in T} \sum_{\vectx \in \xcor_u}  \|\vectx - c(\vectx)\|_1
\geq \sum_{u \in T} \frac{t_u}{2} \| \vectmu^{L,u} - \vectmu^{R,u} \|_1.
$$	
Applying the same steps for the $k$-means cost, we have that
$$
H \cdot \cost(\vectmu^1,\ldots, \vectmu^k) \geq \sum_{u \in T} \sum_{\vectx \in \xcor_u}  \|\vectx - c(\vectx)\|_2^2
\geq \sum_{u \in T} \frac{t_u}{4k} \| \vectmu^{L,u} - \vectmu^{R,u} \|_2^2.
$$
\end{proof}

\subsection{Approximation lower bound}

To complement our upper bounds, we show that a threshold tree with $k$ leaves cannot, in general, yield better than an $\Omega(\log k)$ approximation to the optimal $k$-medians or $k$-means clustering.

\begin{theorem}\label{thm:lb-k}
For any $k \geq 2$, there exists a data set with $k$ clusters such that any  threshold tree $T$ with $k$ leaves
must have $k$-medians and $k$-means cost at least
$$
\cost(T) \geq \Omega(\log k) \cdot \cost(opt).
$$
\end{theorem}

The data set is produced by first picking $k$ random centers from the hypercube $\{-1,1\}^d$, for large enough~$d$, and then using each of these to produce a cluster consisting of the $d$ points that can be obtained by replacing one coordinate of the center by zero.  Thus the clusters have size $d$ and radius $O(1)$. To prove the lower bound, we use ideas from the study of pseudo-random binary vectors, showing that projecting the centers to any subset of $m\lesssim\log_2 k$ coordinates take on all $2^m$ possible values, with each occurring roughly equally often.
 Then, we show that (i) the threshold tree must be essentially a complete binary tree with depth $\Omega(\log_2 k)$ to achieve a clustering with low cost, and (ii) any such tree incurs a cost of $\Omega(\log k)$ times more than the optimal for this data set (for both $k$-medians and $k$-means).
The proof of Theorem~\ref{thm:lb-k} appears in Appendix~\ref{sec:general_k_lower_bound}.

\section{Conclusion}
In this paper we discuss the capabilities and limitations of explainable clusters. For the special case of two clusters ($k=2$), we provide nearly matching upper and lower bounds for a single threshold cut. For general $k >2$, we present the IMM algorithm that achieves an $O(H)$ approximation for $k$-medians and an $O(Hk)$ approximation for $k$-means when the threshold tree has depth $H$ and $k$ leaves. 
We complement our upper bounds with a lower bound showing that any threshold tree with $k$ leaves must have cost at least $\Omega(\log k)$ more than the optimal for certain data sets. 
Our theoretical results provide the first approximation guarantees on the quality of explainable unsupervised learning in the context of clustering. Our work makes progress toward the larger goal of explainable AI methods with precise objectives and provable guarantees.

An immediate open direction is to improve our results for $k$ clusters, either on the upper or lower bound side. One option is to use larger threshold trees with more than $k$ leaves (or allowing more than $k$ clusters). It is also an important goal to identify natural properties of the data that enable explainable, accurate clusters. For example, it would be interesting to improve our upper bounds on explainable clustering for well-separated data. Our lower bound of $\Omega(\log k)$ utilizes clusters with diameter $O(1)$ and separation $\Omega(d)$, where the hardness stems from the randomness of the centers. In this case, the approximation factor $\Theta(\log k)$ is tight because our upper bound proof actually provides a bound in terms of the tree depth (which is about $\log k$, see Appendix~\ref{apx:lower_bound_log_k}). Therefore, an open question is whether a $\Theta(\log k)$ approximation is possible for any well-separated clusters (e.g., mixture of Gaussians with separated means and small variance). Beyond $k$-medians/means, it would be worthwhile to develop other clustering methods using a small number of features (e.g., hierarchical clustering).

\paragraph{Acknowledgements.}
Sanjoy Dasgupta has been supported by NSF CCF-1813160. Nave Frost has been funded by the  European Research  Council (ERC) under the European Unions Horizon 2020 research and innovation programme (Grant agreement No. 804302). The contribution of Nave Frost is part of a Ph.D. thesis research conducted at Tel Aviv University. 

\bibliography{references}
\bibliographystyle{plain}

\newpage
\appendix

\paragraph{Appendix Organization:}
\begin{itemize}
    \item Section~\ref{sec:general_k_lower_bound} contains the $\Omega(\log k)$ lower bound for tree-based clustering using $k$ leaves.
    \item Section~\ref{sec:k2_lower_bound} provides improved lower bounds for $2$-medians or $2$-means for a tree with two leaves.
    \item Section~\ref{sec:k2_upper_bound} contains the remaining upper bound proof exhibiting a 4-approximation for $2$-means.
    \item Section~\ref{sec:DP_Implementation} contains details about improving the running time of the algorithms.
\end{itemize}

\section{Lower Bound: Threshold Tree with exactly $k$ leaves}\label{sec:general_k_lower_bound}

In this section we show that any threshold tree with $k$ leaves must be an $\Omega(\log k)$-approximation, under the $k$-means and $k$-medians cost. We will show a data set that will cause many mistakes. 
This data set consists of $k$ clusters where any two clusters are very far from each other while inside any cluster the points differ by at most two features. 
Each cluster is created by first taking a codeword and then changing one feature at a time to $0.$
The consequence of this process is that for every feature there are many points that globally are very different yet locally all equal to $0$.

The proof of the lower bound has a few steps:
\begin{enumerate}
    \item In Section~\ref{sec:lower_bound_dataset} we show that there is a code such that (i) every two points are far apart, and (ii) when inspecting any $O(\log k)$ features, many codewords are consistent with this local view. From thiscode  we construct our data set with $dk$ points and $\cost(opt)=O(dk).$
    \item In Section~\ref{sec:lower_bound_cluster} we prove that the clusters induced by the threshold tree $T$ are similar to the original clusters, except for at most $k$ points in each cluster. These points will cause $\cost(T)$ to be large. 
    \item In Section~\ref{sec:lower_bound_threshold_tree} we uncover a few properties of any threshold tree created by an  $O(\log k)$-approximation algorithm: up until level $O(\log k)$ the tree has to be complete and no feature is used more than once.  
    \item In  Section~\ref{sec:lower_bound_final_proof} we put together all the claims and show that each level causes $\Omega(k\log k)$ mistakes, each with a cost of $\Omega(d)$, thus $\cost(T)=\Omega(kd\log k)$ which proves the lower bound of $\Omega(\log k)$-approximation. 
\end{enumerate}

\paragraph{Data set construction.} We first take $k$ \emph{codewords} $\mathbf{v}^1,\ldots,\mathbf{v}^k\in\{+1,-1\}^d$ that have the properties described in Claim~\ref{clm:general_k_lower_bound_logk_random}. From each codeword $\mathbf{v}$ we create $d$ data points, $\cX^{\mathbf{v}}$, each time by changing exactly one feature to $0.$ In total we have $dk$ points in the data set, $\cX=\cup_i \cX^{\mathbf{v}^i}$. The cost of the clustering that cluster together all points that belong to the same vector $\mathbf{v}^i$ is $O(dk)$, as the cost of each point is $\Theta(1).$ Thus, $\cost(opt)\leq O(dk).$ 

\subsection{The data set}\label{sec:lower_bound_dataset}
\begin{proposition}[Hoeffding’s inequality] 
Let $X_1, ..., X_n$ be independent random variables, where for each $i$, $X_i\in[0,1]$. Define the random variable $X=\sum_{i=1}^n X_i.$ Then, for any $t\geq 0,$ we have $\Pr(|X-\E[X]|\geq t)\leq 2e^{-\frac{2t^2}{n}}.$
\end{proposition}

\begin{claim}\label{clm:general_k_lower_bound_logk_random}
For any $k\geq 3$, there are $k$ points $C\subseteq\{\pm1\}^d$ that have the following properties for any $\epsilon\geq\frac{\ln(k)}{\sqrt{k}}$: 
\begin{enumerate}
    \item $d = k^3$
    \item for every $\mathbf{c}\neq \mathbf{c}'\in C$ their distance is linear, i.e., $|\{i: c_i\neq c'_i\}|\geq d/4.$ \item for every $\ell\leq\frac{\ln(k)}{50}$ indexes in $[d]$, and every assignment to these indexes, the number of points in $C$ that has these assignment is at least $k(1/2^\ell - \epsilon)$
\end{enumerate}
\end{claim}
\begin{proof}

Take $k$ random points in $\{\pm1\}^d$. We will show that the probability that all properties hold is bigger than $0$ and this will prove our claim using the probabilistic method.

To prove the second property, we use Hoeffding's inequality and union bound. We can bound the probability that any two points in $C$ agree by more than $3d/4$ coordinates by $2k^2e^{-d/8}<1-e^{-1}$ for $k\geq3.$

To prove the third property we again use Hoeffding's inequality and union bound. 
This time though we have $k$ random variables, one for each point. 
There are $\binom{d}{\ell}$ possible $\ell$ coordinates, and there are $2^\ell$ possible assignments to these coordinates. For specific $\ell$ coordinates and an assignment to these coordinates the expected number of points in $C$ that has the specific assignment is $k/2^\ell$. By Hoeffding's inequality, the probability that we deviate by $\epsilon k$ is less than $e^{-\epsilon^2k}$. 
The probability that the last property does not hold is bounded by 
$$\binom{d}{\ell}2^{\ell+1} e^{-2\epsilon^2k}\leq e^{\ell\ln d+2\ell+1-2\epsilon^2 k}.$$ Thus for $\epsilon\geq\frac{\ln(k)}{\sqrt{k}},$ the last term is smaller than $e^{-1}.$

\end{proof}

\subsection{The cluster created by a threshold tree}\label{sec:lower_bound_cluster}
\begin{claim}\label{clm:lower_bound_many_original_points_same_cluster}
 For any threshold tree $T$ with at most $k$ leaves, and for any codeword $\mathbf{v}$, the leaf containing $\mathbf{v}$ also contains at least $d-k$ points of $\cX^{\mathbf{v}}$.
\end{claim}

\begin{proof}
There are at most $k\geq 3$ leaves in $T$, thus in the root to leaf path of the codeword $v$ there are at most $k-1$ features. Hence, all data points in $\cX^{\mathbf{v}}$  that agree on this features must reach the same leaf and be in the same cluster. There are at least $d-k$ such points in $\cX^{\mathbf{v}}$.    
\end{proof}

\begin{claim}\label{clm:lower_bound_many_original_points_different_cluster}
If there are $\alpha$ points from $\cX^{\mathbf{v}^1}$ and $\beta$ points from $\cX^{\mathbf{v}^2}$, $\mathbf{v}^1\neq\mathbf{v}^2$, that are in the same cluster in $T$, then their contribution to $\cost(T)$ is at least  $\frac{1}{4}\min(\alpha,\beta)d.$ The claim holds both under the $\ell_1$ cost and the $\ell_2$ squared cost.  
\end{claim}
\begin{proof}
Denote the center that contains $\alpha$ points from $\cX^{\mathbf{v}^1}$ and $\beta$ points from $\cX^{\mathbf{v}^2}$ by $\vectmu.$ Without loss of generality $\alpha\leq \beta.$ We can disjointly match $\alpha$ points from the two different clusters $(\vectx^1,\vecty^1),\ldots,(\vectx^\alpha,\vecty^\alpha)$, which means that their contribution to $\cost(T)$ is at least $$ \sum_{j=1}^\alpha\norm{\vectx^j-\vectmu}_2^2+\norm{\vectmu-\vecty^j}_2^2 \geq\sum_{j=1}^\alpha\frac12\norm{\vectx^j-\vecty^j}_2^2\geq\frac12\cdot \alpha\cdot \left(\frac{d}4-2\right)\cdot 4
\geq \frac{d\alpha}{4},$$ where the first inequality follows Claim~\ref{clm:cauchy_schwarz_k_means}, the second inequality follows from Claim~\ref{clm:general_k_lower_bound_logk_random} and the fact that if two codewords are different by at least $d/4$ features, then the points differ  by at least $d/4-2$ features, each contributing a cost of $4$, the third inequality follows from the fact that $d=k^3\geq 16$ for $k\geq3.$  
Similarly for the $\ell_1$ cost   $$ \sum_{j=1}^\alpha\norm{\vectx^j-\vectmu}_1+\sum_{j=1}^\alpha\norm{\vectmu-\vecty^j}_1 \geq\norm{\vectx^j-\vecty^j}_1\geq\alpha\cdot \left(\frac{d}4-2\right)\cdot 2\geq \frac{d\alpha}{4}.$$
\end{proof}

\subsection{The threshold tree}\label{sec:lower_bound_threshold_tree}
The next two claims prove that if a feature is used twice or the tree is not complete until level $\frac{\ln(k)}{50}$, then the clustering tree $T$ cannot be an $O(\log k)$-approximation because it shows that $\cost(T)\gtrapprox d^2 \gg \log k\cdot \cost(opt). $
\begin{claim}\label{clm:lower_bound_no_feature_twice}
Fix a threshold tree $T$ with $k \geq 3$ leaves. If there is a feature that is used twice on the same root-to-leaf path in $T$, then $$\cost(T) \geq \frac{d(d-k)}{4}.$$
\end{claim}

\begin{proof}
The proof is composed of two steps, first we show that if a feature is used twice then there is leaf that it is unreachable by a codeword. Then we will show that this implies that two codewords share the same cluster, and thus $\cost(T)$ is high. 

Assume that the there are two nodes in $T$, both of them use the same feature $i$, one with threshold $\theta$ and the other with threshold $\theta'.$ If $\theta=\theta'$ then there is a leaf that is not reachable. Otherwise, the two thresholds divide the line into three parts, and since the codewords have only two values, there is a leaf unreachable by any codeword. Summing up these two cases, there is a leaf that is not reached by any codeword. 

From the pigeonhole principle there are two codewords that share the same cluster which is a contradiction using Claims~\ref{clm:lower_bound_many_original_points_same_cluster} and \ref{clm:lower_bound_many_original_points_different_cluster}. 
\end{proof}

\begin{claim}\label{clm:lower_bound_balanced_till_log_k}
If threshold $T$ contains a leaf at depth less than $\frac{\ln k}{50}$, then $$\cost(T)\geq\frac{d(d-k)}4.$$ The claim is true both under the $k$-means and the $k$-medians cost.
\end{claim}

\begin{proof}
Assume $T$ is not complete until level $\frac{\log k}{50}$. So there is a leaf at a level smaller than  $\frac{\log k}{50}$. By the construction of the data set, there are at least $(\frac{1}{2^\ell}-\epsilon)k> 1$ codewords that reach this leaf. The claim follows from Claims~\ref{clm:lower_bound_many_original_points_same_cluster} and \ref{clm:lower_bound_many_original_points_different_cluster}.
\end{proof}
\subsection{Proof of Theorem~\ref{thm:lb-k}}\label{sec:lower_bound_final_proof}
Assume by contradiction that $T$ is an $O(\log k)$-approximation. From Claim~\ref{clm:lower_bound_many_original_points_same_cluster} we deduce that for each codeword $\mathbf{v}$, at least $d-k$ points from $\cX^{\mathbf{v}}$ will be in the same cluster. From Claim~\ref{clm:lower_bound_many_original_points_different_cluster} and the assumption that $T$ is $O(\log k)$-approximation we get that each $d-k$ such points must be in its own cluster, this cluster will be called the \emph{main cluster} of $\cX^{\mathbf{v}}.$ 

The only values that features can get in the data set are $+1,-1$ or $0.$ Thus, we can assume, without loss of generality, that each threshold is either $0.5$ or $-0.5.$ Focus on some node in $T$ at level $\ell$ with feature $i$ and threshold $\theta$. If $\theta=0.5$, then for all codewords $\mathbf{v}$ with $v_i=1$ the point in $\cX^{\mathbf{v}}$ with $v_i$ will be separated for its main cluster.
From the construction of the data set, there are at least $(\frac{1}{2^\ell}-\epsilon)k$ such points. 
Similarly for $\theta=-0.5,$ we can show that there are at least $(\frac{1}{2^\ell}-\epsilon)k$ such points. 
From Claim~\ref{clm:lower_bound_no_feature_twice}, we deduce that these mistakes are disjoint. 

Applying Claim~\ref{clm:lower_bound_balanced_till_log_k}, there are $2^{\ell-1}$ nodes at each level up until level $\frac{\log k}{50}.$ Hence, total number of mistakes, i.e., points that will not go with their codeword, can be lower bounded by the following using $\epsilon=\frac{\ln(k)}{\sqrt{k}}$ and large enough $k$:   $$\sum_{\ell=1}^{\frac{\log k}{50}}2^{\ell-1}\left(\frac{1}{2^\ell}-\epsilon\right)k\geq\frac{k\log k}{200}.$$

Thus, from Claim~\ref{clm:lower_bound_many_original_points_same_cluster} we can lower bound the cost of $T$: 
$$\cost(T)\geq \frac{kd\log k}{200}=\Omega(\log k)\cost(opt)$$

\subsection{IMM Upper Bound for this dataset}\label{apx:lower_bound_log_k}

We sketch the proof that the IMM algorithm produces a tree of depth $O(\log k)$ for the above dataset construction with high probability. In particular, the upper bound from Theorem~\ref{thm:main-k} is tight for $k$-medians up to the leading constant for this dataset. 

The analysis will follow the standard bound on the maximum clique size in a random graph. Consider fixing any $\ell = 3 \log_2 k$ coordinates to $\pm 1$. When the set of $k$ centers $C$ is chosen uniformly at random from $\{\pm 1\}^d$ and $d = k^3$, we show that with high probability there are at most $\ell$ centers consistent with these values.
When IMM builds the tree, it always chooses a threshold that reduces the number of centers in the children of the current node, and hence, it never splits on the same feature twice. Moreover, it stops the recursion when there is a single center in a leaf. 
Therefore, after 
$3 \log_2 k$ thresholds, the remaining depth of the tree is at most $3 \log_2 k$, and hence, the total depth of the tree is at most $6 \log_2 k$ as well. 

More formally, let $\sigma \in \{\pm 1\}^\ell$ be any sign pattern, and let $C_{I,\sigma}$ be set of centers having pattern $\sigma$ when projected onto coordinates $I \subseteq [d]$ with $|I| = \ell$. 
Then, using the standard upper bound on the binomial coefficient, we have
\begin{eqnarray*}
\Pr\Big[|C_{I,\sigma}| \geq \ell\Big] \leq
\Pr\Big[|C_{I,\sigma}| = \ell \Big] \leq \E\Big[|\{I : |C_{I,\sigma}| = \ell\}| \Big] 
= \binom{d}{\ell} 2^{-\ell^2}
\leq \left(\frac{d e}{\ell}\right)^\ell 2^{-\ell^2} 
= \left(\frac{k^3 e}{2^\ell \ell } \right)^{\ell}.
\end{eqnarray*}
Therefore, plugging in $2^\ell= k^3$, we see that this probability is at most $(e/\ell)^\ell$. Taking a union bound over the $2^\ell$ possible settings of $\sigma \in \{\pm 1\}^\ell$ shows that the probability that there are $\ell$ centers consistent with any $\sigma$ tends to zero as $k$ increases.

\section{Lower bounds for two clusters}\label{sec:k2_lower_bound}

Without loss of generality we can assume that $d\geq 2.$ We use the following dataset for both $2$-medians and $2$-means. It consists of $2d$ points, partitioned into two clusters of size $d$, which are the points with Hamming distance exactly one from the vector with all 1 entries and the vector with all $-1$ entries: 
\begin{center}
\begin{tabular}{c c}
    \textbf{Optimal Cluster 1} & \textbf{Optimal Cluster 2} \\
    $(0,-1,-1,-1\ldots,-1)$ & $(0,1,1,1\ldots,1)$ \\
    $(-1,0,-1,-1\ldots,-1)$ & $(1,0,1,1\ldots,1)$\\
    $(-1,-1,0,-1\ldots,-1)$ & $(1,1,0,1\ldots,1)$\\
    $\vdots$ & $\vdots$ \\
    $(-1,-1,-1,-1\ldots,0)$ & $(1,1,1,1\ldots,0)$\\
\end{tabular}
\end{center}

Let $\widehat C = (\widehat C^1, \widehat C^2)$ be the best threshold cut. 

\paragraph{2-medians lower bound.}
The cost of the cluster with centers $(1,\ldots,1)$ and $(-1,\ldots,-1)$ is $2d$, as each point is responsible for a cost of $1.$ Thus, $\cost(opt)\leq 2d.$

There is a coordinate $i$ and a threshold $\theta$ that defines the cut $\widehat C$. For any coordinate $i$, there are only three possible values: $-1,0,1$. Thus $\theta$ is either in $(-1,0)$ or in $(0,1)$. Without loss of generality, assume that $\theta\in(-1,0)$ and $i=1$. Thus, the cut is composed of two clusters: one of size $d-1$ and the other of size $d+1$, in the following way:

\begin{center}
\begin{tabular}{c c}
    $\mathbf{Cluster\ }\widehat C^1$ & $\mathbf{Cluster\ } \widehat C^2$ \\
    $(-1,0,-1,-1\ldots,-1)$ & $(1,0,1,1\ldots,1)$\\
    $(-1,-1,0,-1\ldots,-1)$ & $(1,1,0,1\ldots,1)$\\
    $\vdots$ & $\vdots$ \\
    $(-1,-1,-1,-1\ldots,0)$ & $(1,1,1,1\ldots,0)$\\
     & $(0,1,1,1\ldots,1)$\\
     & $(0,-1,-1,-1\ldots,-1)$\\
\end{tabular}
\end{center}

Using Fact~\ref{fact:1-median-optimal center}, an optimal center of the first cluster is all $-1$, and the optimal center for the second cluster is all $1$. The cost of the first cluster is $d-1$, as each point costs $1$. The cost of the second cluster is composed of two terms $d$ for all points that include 1 in at least one coordinate and the cost of point $(0,-1,\ldots,-1)$ is $2(d-1)+1$. So the total cost is $4d-2$. Thus $\cost(\widehat C)\geq (2-1/d) \cost(opt).$

\paragraph{2-means lower bound.}
Focus on the clustering with centers $$(\nicefrac{(d-1)}{d},\ldots,\nicefrac{(d-1)}{d}) \qquad \mbox{and} \qquad (-\nicefrac{(d-1)}{d},\ldots,-\nicefrac{(d-1)}{d}).$$ 
The cost of each point in the data is composed of (1) one coordinate with value zero, and the cost of this coordinate is $\left(\nicefrac{(d-1)}{d}\right)^2$ (2) $d-1$ coordinates each with cost $\nicefrac1d^2.$ Thus, each point has a cost of $\nicefrac{(d-1)^2}{d^2}+\nicefrac{d-1}{d^2}.$ Thus, the total cost is $\frac{2(d-1)^2+2(d-1)}{d}=2(d-1)$. This implies that $\cost(opt)\leq 2(d-1).$

Assume without loss of generality that $\widehat C$ is defined using coordinate $i=1$ and threshold $-0.5$. The resulting clusters $\widehat C^1$ and $\widehat C^2$ are as in the case of $2$-medians. The optimal centers are (see Fact~\ref{fact:1-means-optimal center}): 
$$\left(-1,-\frac{d-2}{d-1},\ldots,-\frac{d-2}{d-1}\right) \qquad \mbox{and} \qquad \left(\frac{d-1}{d+1},\frac{d-2}{d+1},\ldots,\frac{d-2}{d+1}\right).$$
We want to lower bound $\cost(\widehat C).$ We start with the cost of the first cluster, i.e. $\widehat C^1$. To do so for each point in $\widehat C^1$, we will evaluate the contribution of each coordinate to the cost (1) the first coordinate adds $0$ to the cost (2) the coordinate with value $0$, adds $\left(\frac{d-2}{d-1}\right)^2$ to the cost (3) the rest of the $d-2$ coordinates adds $\nicefrac{1}{(d-1)^2}.$ Thus, each point in $\widehat C^1$ adds to the cost $\left(\frac{d-2}{d-1}\right)^2 + \frac{d-2}{(d-1)^2}=\frac{d-2}{d-1}$. Since $\widehat C^1$ contains $d-1$ points, its total cost is $d-2$.

Moving on to evaluating the cost of $\widehat C^2$, the cost of the point $(0,-1,\ldots,-1)$ is composed of two terms (1) the first coordinate adds $\left(\frac{d-1}{d+1}\right)^2$ to the cost (2) each of the other $d-1$ coordinates adds $\left(1+\frac{d-2}{d+1}\right)^2$
to the cost. Thus, this point adds $$\left(\frac{d-1}{d+1}\right)^2+(d-1)\left(1+\frac{d-2}{d+1}\right)^2=\frac{(d-1)d(4d-3)}{(d+1)^2}.$$
Similarly, the point $(0,1,\ldots,1)$ adds to the cost $$\left(\frac{d-1}{d+1}\right)^2+(d-1)\left(1-\frac{d-2}{d+1}\right)^2=\frac{(d-1)(d+8)}{(d+1)^2}.$$
Finally, each of the $d-1$ remaining points in $\widehat C^2$ adds to the cost $$\left(1-\frac{d-1}{d+1}\right)^2+\left(\frac{d-2}{d+1}\right)^2+(d-1)\left(1-\frac{d-2}{d-1}\right)^2=\frac{d^2+5d-1}{(d+1)^2}$$
Thus, the cost of $\widehat C^2$ is $$\frac{(d-1)(5d^2+3d+7)}{(d+1)^2}$$
Summing up the costs of $\widehat C^1$ and $\widehat C^2$, for $d\geq 2$
$$\cost(\widehat C)\geq (d-2)+\frac{(d-1)(5d^2+3d+7)}{(d+1)^2}\geq6(d-1)\left(1-\frac1d\right)^2\geq 3\left(1-\frac{1}{d}\right)^2\cdot \cost(opt)$$


\section{Upper Bound Proof for 2-Means}\label{sec:k2_upper_bound}

We show that there is a threshold cut $\widehat C$ with $2$-means cost satisfying
$\cost(\widehat C)\leq 4 \cdot \cost(opt).$ We could just use the same proof idea as in the 2-medians case that first applies Lemma~\ref{lemma:tree-cost} and then uses the matching result, Lemma~\ref{lemma:matching}. This leads to a $6$-approximation, instead of $4$. The reason is that we apply twice Claim~\ref{clm:cauchy_schwarz_k_means}, which is not tight. Improving the approximation to $4$ requires us to apply   Claim~\ref{clm:cauchy_schwarz_k_means} only once. 

Suppose $\vectmu^1,\vectmu^2$ are optimal $2$-means centers for the clusters $C^1$ and $C^2$. 
Let $t = \min(|C^1 \Delta \widehat C^1|, |C^1 \Delta \widehat C^2|)$ be the minimum number of changes for any threshold cut $\widehat C^1, \widehat C^2$, and define $\xmis$ to the set of $t$ points in the symmetric difference, where $\cX = \xcor \cup \xmis$ and $\xcor \cap \xmis = \emptyset$.

Using the same argument as in the proof of Lemma~\ref{lemma:tree-cost}, we have
\begin{eqnarray}\label{eq:2_means_upper_bound}
\cost(\widehat{C}) &\leq& \sum_{j=1}^2 \sum_{\vectx \in \xcor \cap \widehat C_j} \|\vectx- \vectmu^j\|^2_2 +
\sum_{j=1}^2 \sum_{\vectx \in \xmis \cap \widehat C_j} \|\vectx- \vectmu^j\|^2_2 \nonumber
\\ &=& \sum_{\vectx \in \xcor} \|\vectx- c(\vectx)\|^2_2 +
\sum_{\vectx \in \xmis \cap \widehat C_1} \|\vectx- \vectmu^1\|^2_2+\sum_{\vectx \in \xmis \cap \widehat C_2} \|\vectx- \vectmu^2\|^2_2 \nonumber
\\ &\leq& \cost(opt) +
\sum_{\vectx \in \xmis \cap \widehat C_1} \|\vectx- \vectmu^1\|^2_2+\sum_{\vectx \in \xmis \cap \widehat C_2} \|\vectx- \vectmu^2\|^2_2
\end{eqnarray}
The goal now is to bound the latter two terms using $\cost(opt).$
This term measures the distance of each $\vectx \in \xmis$ from the ``other'' center, i.e., not $c(\vectx)$. 

\begin{claim}\label{claim:aux-2-means}
\begin{eqnarray*}
	\cost(opt)
	\geq \frac13 \sum_{\vectx \in \xmis \cap \widehat C_1} \|\vectx- \vectmu^1\|^2_2+\frac13\sum_{\vectx \in \xmis \cap \widehat C_2} \|\vectx- \vectmu^2\|^2_2
\end{eqnarray*}
\end{claim}

Using Claim~\ref{claim:aux-2-means},
together with Inequality~(\ref{eq:2_means_upper_bound}) we have
$$\cost(\widehat{C})\leq \cost(opt) + 3\cdot \cost(opt)=4\cdot \cost(opt),$$ 
and this completes the proof.

\begin{proof}[Proof of Claim~\ref{claim:aux-2-means}.]
Denote the $t$ points in $\xmis$  by $\xmis=\{\mathbf{r}^1,\ldots,\mathbf{r}^t\}.$ Assume that the first $\ell$ points are in the first optimal cluster, $\mathbf{r}^1,\ldots,\mathbf{r}^\ell\in C^1$, and the rest are in the second cluster, $\mathbf{r}^{\ell+1},\ldots,\mathbf{r}^t\in C^2.$

Applying Lemma~\ref{lemma:matching} for each coordinate $i\in[d]$ guarantees $t$ pairs of vectors $(\mathbf{p}^1,\mathbf{q}^1), \ldots, (\mathbf{p}^t,\mathbf{q}^t)$ with the following properties. Each $p^j_i$ corresponds to the $i$th coordinate of some  point in $C^1$ and $q^j_i$ corresponds to the $i$th coordinate of some point in $C^2$. Furthermore, for each coordinate, the $t$ pairs correspond to $2t$ distinct points in $\cX$.
Finally, we can assume without loss of generality that
$\mu^1_i \leq \mu^2_i$ and $q^j_i \leq p^j_i$. 

For each point $\mathbf{r}^j$ in the first $\ell$ points in $\xmis$, if $r^j_i\geq p^j_i$ then we can replace $\mathbf{p}^j$ with $\mathbf{r}^j$, thus we can assume without loss of generality that $p_i^j\geq r^j_i.$ We next show that $\cost(opt)$ is lower bounded by  a function of $t$. There will be two cases depending on whether $p^j_i\leq \mu^2_i$ or not. The harder case is the first where the improvement of the approximation from $6$ to $4$ arises. Instead of first bounding the distance between $\mathbf{r}^j$ and its new center using the distance to its original center and then accounting for $\norm{\vectmu^1-\vectmu^2}^2_2,$ we directly account for the distance between $\mathbf{r}^j$ and its new center.

\paragraph{\textbf{Case 1:}} if  $p^j_i\leq \mu^2_i$, then 
Claim~\ref{clm:cauchy_schwarz_k_means} implies that
\begin{eqnarray*} 
(\mu^2_i-q^j_i)^2 + (p^j_i-\mu_i^1)^2 + (\mu_i^1-r^j_i)^2
&\geq& \frac{1}{3}(\mu^2_i-q^j_i+p^j_i-\mu_i^1+\mu_i^1-r^j_i)^2\\
&=& \frac{1}{3}((\mu^2_i-q^j_i)+(p^j_i-r^j_i))^2
\geq \frac{1}{3}(\mu^2_i-r^j_i)^2.
\end{eqnarray*}  The last inequality follows from $q^j_i \leq p^j_i$ and $r_i^j\leq p_i^j,$ which imply that $(\mu^2_i-q^j_i)+(p^j_i-r^j_i)
\geq \mu^2_i-r^j_i\geq 0,$ which means $((\mu^2_i-q^j_i)+(p^j_i-r^j_i))^2
\geq (\mu^2_i-r^j_i)^2.$ 
\paragraph{\textbf{Case 2:}} if $\mu_i^2\leq p_i^j$, then again Claim~\ref{clm:cauchy_schwarz_k_means} implies that
\begin{eqnarray*} 
(p^j_i-\mu_i^1)^2 + (\mu_i^1-r^j_i)^2&\geq& (\mu_i^2-\mu_i^1)^2 + (\mu_i^1-r^j_i)^2 \geq \frac12 (\mu^2_i-\mu^1_i+\mu^1_i-r^j_i)^2=\frac12(\mu^2_i-r^j_i)^2,
\end{eqnarray*}
where in the first inequality we use $(p^j_i-\mu^1_i)^2\geq(\mu^2_i-\mu^1_i)^2.$

The two cases imply that for $1\leq j\leq \ell$  $$(\mu^2_i-q^j_i)^2 + (p^j_i-\mu_i^1)^2 + (\mu_i^1-r^j_i)^2
\geq  \frac{1}{3}(\mu^2_i-r^j_i)^2
.$$
Similarly for each point $\mathbf{r}^j$ in the last $t-\ell$ points in $\xmis$, we have $$(\mu^2_i-q^j_i)^2 + (p^j_i-\mu_i^1)^2 + (\mu_i^2-r^j_i)^2
\geq  \frac{1}{3}(\mu^1_i-r^j_i)^2
.$$
Putting these together we have 
\begin{eqnarray*}
	\cost(opt)&\geq& \sum_{i=1}^d\sum_{j=1}^\ell(\mu^2_i-q^j_i)^2 + (p^j_i-\mu_i^1)^2 + (\mu_i^1-r^j_i)^2
	+
\sum_{i=1}^d\sum_{j=\ell+1}^t(\mu^2_i-q^j_i)^2 + (p^j_i-\mu_i^1)^2 + (\mu_i^2-r^j_i)^2\\
	&\geq& \frac{1}{3}\sum_{j=1}^\ell\sum_{i=1}^d  (\mu^2_i-r^j_i)^2
	+\frac{1}{3}\sum_{j=\ell+1}^t\sum_{i=1}^d  (\mu^1_i-r^j_i)^2 
	\\ 	&=&\frac{1}{3}\sum_{j=1}^\ell\norm{\mathbf{r}^j-\vectmu^2}_2^2+\frac{1}{3}\sum_{j=\ell+1}^t\norm{\mathbf{r}^j-\vectmu^1}_2^2\\
&=&	\frac13 \sum_{\vectx \in \xmis \cap \widehat C_1} \|\vectx- \vectmu^1\|^2_2+\frac13\sum_{\vectx \in \xmis \cap \widehat C_2} \|\vectx- \vectmu^2\|^2_2
\end{eqnarray*}
\end{proof}


\section{Efficient Implementation via Dynamic Programming for $k=2$}\label{sec:DP_Implementation}

\subsection{The 2-means case}
The psudo-code for finding the best threshold for $k=2$ depicted in Algorithm~\ref{algo:cut}.

\begin{figure}[!htb]
  \centering
  \begin{minipage}{.6\linewidth}
        \begin{algorithm}[H]
    	\SetKwFunction{$2$-means Optimal Threshold}{$2$-means Optimal Threshold}
    	\SetKwInOut{Input}{Input}\SetKwInOut{Output}{Output}
    	\LinesNumbered
    	\Input{%
    		{$\vectx^1, \ldots, \vectx^n$} -- vectors in $\reals^d$
    	}
    	\Output{%
    		\xvbox{2mm}{$i$} -- Coordinate\\
    		\xvbox{2mm}{$\theta$} -- Threshold
    	}
    	\BlankLine
    	\setcounter{AlgoLine}{0}
    	
        $\xvar{best\_cost} \leftarrow \infty$\;
    	
    	$\xvar{best\_coordinate} \leftarrow \NULL$\;
    	
    	$\xvar{best\_threshold} \leftarrow \NULL$\;
    	
    	$u \leftarrow \sum_{j=1}^n \norm{\vectx^j}^2_2$\;
    	
    	\ForEach {$i \in [1, \ldots, d]$}
    	{
    		$\mathbf{s} \leftarrow \xfunc{zeros}(d)$\;
    		
    		$\mathbf{r} \leftarrow \sum_{j=1}^n \vectx^j$\;
    		
    		$\mathcal{X} \leftarrow \xfunc{sorted}(\vectx^1, \ldots, \vectx^n \text{ by coordinate }i)$\;
    		
    		\ForEach {$\vectx^j \in \mathcal{X}$} 
    		{ 
    			$\mathbf{s} \leftarrow \mathbf{s} + \vectx^j$\;
    			
    			$\mathbf{r} \leftarrow \mathbf{r} - \vectx^j$\;
    			
    			$\xvar{cost} \leftarrow u - \frac{1}{j}\norm{\mathbf{s}}^2_2- \frac{1}{n - j}\norm{\mathbf{r}}^2_2$\;
    			
    			\If {$\xvar{cost} < \xvar{best\_cost}$ and $x^j_i \neq x^{j+1}_i$}
    			{

    				$\xvar{best\_cost} \leftarrow \xvar{cost}$\;
    				
    				$\xvar{best\_coordinate} \leftarrow i$\;
    				
    				$\xvar{best\_threshold} \leftarrow x^j_i$\;
    			}
    		}
    	}
    	\Return $\xvar{best\_coordinate}, \xvar{best\_threshold}$\;
    	\caption{\textsc{Optimal Threshold for $2$-means}}
    	\label{algo:cut}
    \end{algorithm}
    \end{minipage}
\end{figure}

In time $O(d)$ we can calculate $\cost(p+1)$ and the new centers by using the value $\cost(p)$ and the previous centers. Throughout the computation we save in memory 
\begin{enumerate}
	\item Two vectors $\mathbf{s}^{p}=\sum_{j=1}^p \vectx^j$ and $\mathbf{r}^{p}=\sum_{j=p + 1}^n \vectx^j$.
	\item Scalar $u=\sum_{j=1}^n \norm{\vectx^j}_2^2$
\end{enumerate}
We also make use of the identity:
\begin{eqnarray*}
	\cost(p) 
	&=& u - \frac{1}{p}\norm{\mathbf{s}^{p}}_2^2 - \frac{1}{n-p}\norm{\mathbf{r}^{p}}_2^2.
\end{eqnarray*}
This identity is correct because 
\begin{eqnarray*}
    \cost(p) &=& \sum_{j=1}^{p}\norm{\vectx^j-\vectmu^1(p)}^2_2 + \sum_{j=p+1}^{n}\norm{\vectx^j-\vectmu^2(p)}^2_2\\
    &=& \sum_{j=1}^{p} \norm{\vectx^j}^2_2 - 2\sum_{j=1}^{p}\inner{\vectx^j}{\vectmu^1(p)} + \sum_{j=1}^{p}\norm{\vectmu^1(p)}^2_2 + \\
    && \sum_{j=p+1}^{n} \norm{\vectx^j}^2_2 - 2\sum_{j=p+1}^{n}\inner{\vectx^j}{\vectmu^2(p)} + \sum_{j=p+1}^{n}\norm{\vectmu^2(p)}^2_2\\
    &=& \sum_{j=1}^{n} \norm{\vectx^j}^2_2 - 2\inner{\sum_{j=1}^{p}\vectx^j}{\vectmu^1(p)} + \frac{1}{p}\norm{\sum_{j=1}^{p}\vectx^j}^2_2 - \\
    && 2\inner{\sum_{j=p+1}^{n}\vectx^j}{\vectmu^2(p)} + \frac{1}{n-p}\norm{\sum_{j=p+1}^{n}\vectx^j}^2_2\\
    &=& \sum_{j=1}^{n} \norm{\vectx^j}^2_2 - \frac{2}{p}\inner{\mathbf{s}^{p}}{\mathbf{s}^{p}}+\frac{1}{p}\norm{\mathbf{s}^{p}}^2_2 - \frac{2}{n-p}\inner{\mathbf{r}^{p}}{\mathbf{r}^{p}}+\frac{1}{n-p}\norm{\mathbf{r}^{p}}^2_2\\
    &=& u - \frac{1}{p}\norm{\mathbf{s}^{p}}^2_2 - \frac{1}{n-p}\norm{\mathbf{r}^{p}}^2_2
\end{eqnarray*}

 By invoking this identity, we can quickly compute the cost of placing the first $p$ points in cluster one and the last $n-p$ points in cluster two. Each such partition can be achieved by using a threshold $\theta$ between $x_i^p$ and $x_i^{p+1}$. Our algorithm computes these costs for each feature $i \in[d]$. Then, we output the feature $i$ and threshold $\theta$ that minimizes the cost. This guarantees that we find the best possible threshold cut.

 Overall, Algorithm \ref{algo:cut} iterates over the $d$ features, and for each feature it sorts the $n$ vectors according to their values in the current feature. Next, the algorithm iterates over the $n$ vectors and for each potential threshold, it calculates the cost by evaluating the inner product of two $d$-dimensional vectors.
 Overall its runtime complexity is $O\left(nd^2 + nd\log n\right)$.
 
 \subsection{The 2-medians case}
 The high level idea of a finding an optimal 2-medians cut is similar to the 2-means algorithm. The algorithm goes over all possible thresholds. For each threshold, it finds the optimal centers and calculates the cost accordingly. Then, it outputs the threshold cut that minimizes the $2$-medians cost.
 
  \begin{figure}[!htb]
  \centering
  \begin{minipage}{.6\linewidth}
        \begin{algorithm}[H]
    	\SetKwFunction{$2$-medians Optimal Threshold}{$2$-means medians Threshold}
    	\SetKwInOut{Input}{Input}\SetKwInOut{Output}{Output}
    	\LinesNumbered
    	\Input{%
    		{$\vectx^1, \ldots, \vectx^n$} -- vectors in $\reals^d$
    	}
    	\Output{%
    		\xvbox{2mm}{$i$} -- Coordinate\\
    		\xvbox{2mm}{$\theta$} -- Threshold
    	}
    	\BlankLine
    	\setcounter{AlgoLine}{0}
    	
    	$\xvar{best\_cost} \leftarrow \infty$\;
    	
    	$\xvar{best\_coordinate} \leftarrow \NULL$\;
    	
    	$\xvar{best\_threshold} \leftarrow \NULL$\;
    	
    	\ForEach {$i \in [1, \ldots, d]$}
    	{
    	
    	    $\vectmu^2(0) \leftarrow \xfunc{median}(\vectx^1, \ldots \vectx^n)$\;
    	
    	    $\xvar{cost} \leftarrow \sum_{j=1}^n \norm{\vectx^j - \vectmu^2(0)}_1$\;
    	
    	    $\mathcal{X} \leftarrow \xfunc{sorted}(\vectx^1, \ldots, \vectx^n \text{ by coordinate }i)$\;
    	
    		\ForEach {$j \in [1, \ldots, n - 1]$} 
    		{ 
	    		$\vectmu^1(j) \leftarrow \xfunc{median}(\vectx^1, \ldots \vectx^j)$\;
    		
	    		$\vectmu^2(j) \leftarrow \xfunc{median}(\vectx^{j+1}, \ldots \vectx^n)$\;
	    		
    			$\xvar{cost} \leftarrow \xvar{cost} + \norm{\vectx^j - \vectmu^1(j)}_1 - \norm{\vectx^j - \vectmu^2(j - 1)}_1$\;
    			
    			\If {$\xvar{cost} < \xvar{best\_cost}$ and $x^j_i \neq x^{j+1}_i$}
    			{
    				
    				$\xvar{best\_cost} \leftarrow \xvar{cost}$\;
    				
    				$\xvar{best\_coordinate} \leftarrow i$\;
    				
    				$\xvar{best\_threshold} \leftarrow x^j_i$\;
    			}
    		}
    	}
    	\Return $\xvar{best\_coordinate}, \xvar{best\_threshold}$\;
    	\caption{\textsc{Optimal Threshold for $2$-medians}}
    	\label{algo:2-medians-cut}
    \end{algorithm}
    \end{minipage}
\end{figure}
 
 \paragraph{Updating cost.} 
 To update the cost we need to show how to express $\cost(p+1)$ in terms of $\cost(p).$ We know that $\cost(p+1)$ is equal to 
 $$\cost(p+1)=\sum_{\vectx\in C_1}\norm{x-\vectmu^1(p+1)}_1+\sum_{\vectx\in C_2}\norm{x-\vectmu^2(p+1)}_1.$$ 
 For every feature $i \in[d]$, there are $n-1$ thresholds to consider. After sorting by this feature, we can consider all splits into $C_1$ and $C_2$, where $C_1$ contains the $p$ smallest points,  and $C_2$ contains the $n-p$ largest points. We increase $p$ from $p=1$ to $p=n-1$, computing the clusters and cost at each step. If $p$ is odd then the median of $C_1$ (i.e., the optimal center of $C_1$) does not change compared to $p-1$. The only contribution to the cost is the point~$\vectx$ that moved from $C_2$ to $C_1$. If $p$ is even, then at each coordinate there are two cases, depending on whether the median changes or not. If it changes, then let $\Delta$ denote the change in cost of the points in $C_1$ that are smaller than the median. By symmetry, the change in the cost of the points that are larger is $-\Delta$. Thus, the change of the cost is balanced by the points that are larger and smaller than the median. Similar reasoning holds for the other cluster $C_2.$ Therefore, we conclude that moving $\vectx$ from $C_2$ to $C_1$ changes the cost by exactly $\norm{\vectx-\vectmu^1(p+1)}_1 - \norm{\vectx-\vectmu^2(p)}_1$. Thus, we have the following connection between $\cost(p+1)$ and $\cost(p)$:
 $$\cost(p+1)=\cost(p) + \norm{\vectx-\vectmu^1(p+1)}_1 - \norm{\vectx-\vectmu^2(p)}_1.$$
 
  \paragraph{Updating centers.} For each $p$, the cost update relies on efficient calculations of the centers $\vectmu^1(p)$ and  $\vectmu^2(p + 1)$. The centers $\vectmu^1(p), \vectmu^2(p)$ are the medians of the clusters at the $p$th threshold. Note that moving from the $p$th thresold to the $(p+1)$th will only change the clusters by moving one vector from one cluster to the other. 
  We can determine the changes efficiently by using $d$ arrays, one for each coordinate. Each array will contain (pointers to) the input vectors $\cX$ sorted by their $i$th feature value. As we move the threshold along a single coordinate, we can read off the partition into two clusters, and we can compute the median of each cluster by considering the midpoint in the sorted list. 
  
  Overall, this procedure computes the cost of each threshold, while also determining the partition into two clusters and their centers (medians). The time is $O(nd \log n)$ to sort by each feature, and $O(nd^2)$ to compute $\cost(p)$ for each $p \in [n]$ and each feature. Therefore, the total time for the $2$-medians algorithm is 
  $O(nd^2 + nd \log n).$

\end{document}